\documentclass[11pt]{article}

\PassOptionsToPackage{table}{xcolor}

\usepackage[final]{acl}

\usepackage{times}
\usepackage{latexsym}
\usepackage{amsmath,amssymb,amsthm}
\usepackage{mathtools}
\usepackage{booktabs}
\usepackage{algorithm}
\usepackage{algpseudocode}
\usepackage{hyperref}
\usepackage{cleveref}
\usepackage{enumitem}

\newtheorem{theorem}{Theorem}[section]

\newtheorem{proposition}[theorem]{Proposition}

\newtheorem{lemma}[theorem]{Lemma}

\newtheorem{corollary}[theorem]{Corollary}

\newtheorem{definition}[theorem]{Definition}

\newcommand{\D}{\mathcal{D}}

\newcommand{\Q}{\mathcal{Q}}
\newcommand{\eps}{\varepsilon}
\newcommand{\vbar}{\bar{v}}
\newcommand{\vtil}{\tilde{v}}

\makeatletter
\renewcommand{\theHALG@line}{\thealgorithm.\arabic{ALG@line}}
\makeatother

\usepackage[T1]{fontenc}

\usepackage[utf8]{inputenc}

\usepackage{microtype}
\usepackage{comment}

\usepackage{inconsolata}

\usepackage{graphicx}

\title{Understanding Quantization of Optimizer States in LLM Pre-training: Dynamics of State Staleness and Effectiveness of State Resets}

\author{Kristi Topollai \\
  New York University \\ 
  \texttt{kt2664@nyu.edu} \\\And
  Anna Choromanska \\
  New York University \\
  \texttt{ac5455@nyu.edu} \\}

\begin{document}
\maketitle

\begin{abstract}
Quantizing optimizer states is becoming an important ingredient of memory-efficient large-scale pre-training, but the resulting optimizer dynamics remain only partially understood. We study low-precision exponential moving average (EMA) optimizer states and show how quantization can cause many nominal updates to round back to the same stored value, making the state effectively stale and slowing adaptation beyond what the nominal decay would suggest. We then develop a simple predictive model of stalling that estimates one-step stalling probabilities and characterizes how stalling builds up over time after the initialization. This perspective provides a mechanistic explanation for why optimizer-state resets help in low precision: once a quantized EMA becomes effectively stale, resetting it can temporarily restore responsiveness. Motivated by this picture, we derive a simple theory-guided method for choosing useful reset periods, showing that in low precision the key question is not only whether resets help, but when they should be applied. Experiments in controlled simulations and LLM pre-training show that suitable reset schedules recover the performance lost to low-precision state storage while substantially reducing optimizer-state memory.
\end{abstract}

\section{Introduction}
\label{sec:intro}

Modern large-scale pre-training relies heavily on optimizers such as AdamW
\cite{kingma2017adammethodstochasticoptimization, loshchilov2018decoupled}, Lion \cite{chen2023symbolic}, Shampoo \cite{shampoo}, SOAP
\cite{soap}, and Muon \cite{liu2025muonscalablellmtraining}, which maintain
running statistics, typically exponential moving averages (EMAs), to smooth
gradients, adapt step sizes, or construct preconditioners. At the same time,
modern training pipelines increasingly store these states in reduced precision
to lower memory cost and improve throughput
\cite{dettmers20218,peng2023fp8lmtrainingfp8large,xi2024coat,wang20244}. This
makes optimizer-state quantization practically important, but its effect on the
actual dynamics of the optimizer remains insufficiently understood.

A related empirical observation is that resetting optimizer states can
improve pre-training
\cite{huang2025spam,huang2025stablespam,glentis2025scalable}. These resets are
usually motivated through spike mitigation or generic training stability.
However, in our experiments their benefit is strongly precision-dependent:
resets help substantially when EMA states are stored in BF16 or lower
precision, but are often much less useful in full precision. This suggests
that state quantization and state resets should not be viewed separately:
understanding why quantized states become ineffective is also the key to
understanding when resets help.

In this paper, we study that connection through the dynamics of quantized EMAs.
Consider a generic recursion
\[
x_t=\beta x_{t-1}+(1-\beta)s_t,
\]
where $x_t$ is an optimizer state, $\beta$ is a decay coefficient close to
$1$, and $s_t$ denotes the new update information. This form captures the
state evolution of many widely used optimizers. In low precision, each step
first forms a higher-precision update and then maps it back to the storage
format. When the increment is too small relative to the local spacing of the
floating-point grid, the update rounds back to the same stored value, so the
state does not change. We refer to this phenomenon as \emph{state-update
stalling}, also known as signal swamping and unchanged-state
dynamics in low-bit EMA updates
\cite{higham1993accuracy,wang2018training,xu2025pushinglimitslowbitoptimizers}.
In high-decay recursions, repeated stalling slows adaptation, effectively
lengthens the memory of the EMA beyond what the nominal decay suggests, and
makes quantized states substantially less responsive than their full-precision
counterparts.

\begin{figure*}[t]
    \centering
    \includegraphics[width=\linewidth]{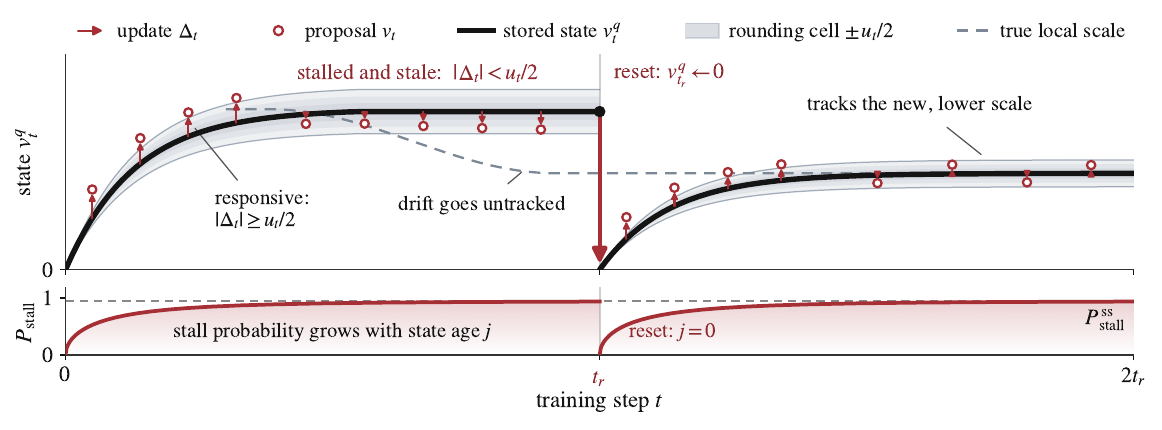}
     \vspace{-1em}
    \caption{Mechanism of state-update stalling and reset recovery in a quantized EMA under nearest rounding. \textbf{Top:} The high-precision proposal $v_t$ changes the stored state $v_t^q$ only when $|\Delta_t|\ge u_t/2$, escaping its local rounding cell. As $v_t^q$ grows, the spacing $u_t=\operatorname{ulp}(v_{t-1}^q)$ also grows, causing more proposals to round back to the stored value and preventing the state from tracking changes in the local scale. \textbf{Bottom:} The resulting stalling probability increases with the state age $j$ toward $P_{\mathrm{stall}}^{\mathrm{ss}}$. Resetting $v_{t_r}^q\leftarrow0$ and $j\leftarrow0$ restores responsiveness, allowing the state to track the new, lower scale.}
     \vspace{-1em}
    \label{fig:intro}
\end{figure*}

This viewpoint also helps explain several recurring features of low-precision
training. In particular, it clarifies why large $\beta_2$ values become more
problematic as precision decreases, why stochastic rounding
\cite{gupta2015deep,Ozkara2025StochasticRF,croci2022stochastic} helps but does
not fully remove the problem when stalling is severe, and why the first and
second moments of Adam-like methods respond differently to quantization. Most
importantly, it explains why resets can help and why their timing
matters: once a quantized EMA has become effectively stale, resetting it can
restore responsiveness, but resetting too early discards useful averaging. The
key question is therefore not only whether resets help, but when they should be
applied.

We formalize these effects through a simple predictive model of quantized EMA
dynamics. The model combines local floating-point grid geometry with a
short-horizon approximation to estimate one-step stalling probabilities and to
characterize the finite post-initialization period during which a quantized EMA remains
meaningfully responsive. This responsive \emph{startup window} provides a
natural basis both for understanding when resets should help and for choosing
useful reset periods as a function of decay and precision. We validate the
resulting predictions in controlled simulations and LLM pre-training, showing
that suitable reset schedules recover the performance lost to
low-precision state storage while substantially reducing optimizer-state
memory.

\section{Related Work}

\paragraph{Optimizers for large-scale pre-training.}
Modern LLM pre-training increasingly follows compute-optimal scaling and open
training recipes~\cite{hoffmann2022empirical,llama1,llama2,groeneveld-etal-2024-olmo},
most of which rely on optimizers with persistent state. Beyond momentum SGD and
Adam/AdamW~\cite{kingma2017adammethodstochasticoptimization, loshchilov2018decoupled}, this includes memory-reduced adaptive methods such as
Adafactor~\cite{adafactor}, SM3~\cite{anil2019memory} and Adam-mini~\cite{zhang2024adam}, structured
preconditioners such as Shampoo~\cite{shampoo} and SOAP~\cite{soap}, and more
recent alternatives such as Lion~\cite{chen2023symbolic} and
Muon~\cite{liu2025muonscalablellmtraining}.

\paragraph{Low-precision and quantized pre-training.}
Recent low-precision pre-training builds on FP8 formats~\cite{Micikevicius2022FP8FF},
with systems such as FP8-LM~\cite{peng2023fp8lmtrainingfp8large} and
COAT~\cite{xi2024coat} demonstrating end-to-end FP8 LLM training, and more
recent work extending to FP4 through mixed-precision or native low-bit
stacks~\cite{pmlr-v267-wang25ae,chmiel2025fp4,castro2026quartetnativefp4training}. Complementary work targets model weights; for example, DB-LLM uses dual
binarization for ultra-low-bit LLM weight
quantization~\cite{chen-etal-2024-db}.
Closest to our setting is direct quantization of optimizer states: 8-bit
and 4-bit states for Adam-like methods~\cite{dettmers20218,4bit_states},
4-bit Shampoo preconditioners~\cite{wang20244}, and 8-bit quantized Muon
states~\cite{gupta2025effective}. Our contribution is complementary: rather
than proposing a new quantization or scaling scheme, we analyze a specific failure mode of
quantized EMA states, state-update stalling under requantization, and use it
to explain when stochastic rounding and periodic resets help.

\paragraph{Optimizer state resetting.}
Restarting or state resetting have a long history in optimization~\cite{ODonoghue2012AdaptiveRF}
and were later adapted to neural network training~\cite{srsgd}. More recent
work shows that clearing optimizer states can help in non-stationary settings
such as reinforcement learning~\cite{rl_restarts} as well as low communication distributed optimization~\cite{topollai2026outermomentum}. In LLM training, methods
such as SPAM~\cite{huang2025spam} and
Stable-SPAM~\cite{huang2025stablespam} use momentum resets for spike suppression and
stabilization, while \citet{topollaiadaptive} show that implicit resets via
adapting the momentum coefficient can also improve performance. In contrast, we
show that in low precision the benefit of resets is instead tied to
quantized-state stalling: resets help by temporarily restoring responsiveness
once EMA states become effectively stale.

\section{State update stalling under quantized EMA states}
\label{sec:stalling}

\paragraph{Preliminaries.}
We consider exponential moving averages (EMAs) of the form
\begin{equation}
x_t = \beta x_{t-1} + (1-\beta)s_t,
\label{eq:ema-generic}
\end{equation}
where $s_t$ is a stochastic signal such as the gradient (e.g., $g_t$ or $g_t^2$) and
$\beta\in(0,1)$ is the decay coefficient. In the quantized setting, the EMA
state is stored in low precision but updated in high. Let $x_t^q$ denote the
stored state, let $D(\cdot)$ denote dequantization (equivalently, an exact cast
to FP32), and let $Q(\cdot)$ denote quantization back to the storage format.
One update therefore reads $x_{t-1}=D(x_{t-1}^q)$, forms the FP32
update $x_t=\beta x_{t-1}+(1-\beta)s_t$, and writes back
$x_t^q=Q(x_t)$. Since $D(\cdot)$ is just a cast, no new quantization error is
introduced during the read phase; all new errors enter on the write phase
through $Q(\cdot)$.

\begin{definition}[State update stalling]
\label{def:state-update-stalling}
We say that state update stalling occurs at step $t$ if
\begin{equation}
x_t^q=x_{t-1}^q.
\label{eq:stall-def}
\end{equation}
Equivalently, the high precision update fails to move the stored state to a new
representable value.
\end{definition}

Quantization therefore acts as a nonlinear gate in the EMA recursion,
suppressing sufficiently small updates. We focus on the standard Adam setting,
where the optimizer maintains first- and second-moment EMAs of the gradient
$g_t=\nabla f(\theta_t)$, namely the momentum
$m_t=\beta_1 m_{t-1}+(1-\beta_1)g_t$ and the second moment estimate
$v_t=\beta_2 v_{t-1}+(1-\beta_2)g_t^2$. Although the same gating mechanism
applies to both moments, it is typically more pronounced for the second
moment, whose quantization effects
are often more visible in practice~\cite{fishman2024scaling,tang2025convergence,collage}. Moreover, the second
moment admits a simple probabilistic model that matches empirical behavior
extremely well. We therefore focus the analysis on the second-moment EMA.

\subsection{One-step stalling as a geometry problem}
\label{sec:stall-geometry}

We analyze a single coordinate of the second-moment EMA: 
\[
v_t = \beta_2 v_{t-1} + (1-\beta_2)g_t^2.
\]
Our analysis is coordinate-wise and does not depend on the global training
objective; it requires only a short-window approximation for one gradient
coordinate. Concretely, over the window of interest we assume that the
coordinate has an approximately fixed local scale $\sigma$, and we analyze the
EMA in the regime where $v_{t-1}$ is of the same order as $\sigma^2$. To obtain
closed-form expressions, we further use the local Gaussian approximation
\begin{equation}
\frac{g_t^2}{\sigma^2}\overset{d}{\approx}\chi_1^2,
\qquad\text{equivalently}\qquad
\frac{g_t}{\sigma}\overset{d}{\approx}\mathcal{N}(0,1).
\label{eq:chi-model}
\end{equation}
This is a tractable short-window approximation for the step-wise distribution
of a single gradient coordinate, and is in line with recent evidence that
stochastic gradients exhibit substantial Gaussian structure in this sense
\cite{xie2023overlooked}.

\noindent The high precision update increment is then
\begin{equation}
\Delta_t := v_t - v_{t-1}
= (1-\beta_2)\bigl(g_t^2-v_{t-1}\bigr).
\label{eq:stall-delta}
\end{equation}

Let $u_t:=\operatorname{ulp}(v_{t-1}^q)$ denote the local spacing between
adjacent representable numbers around the stored value. Under nearest rounding,
state-update stalling occurs whenever the high precision proposal remains within the
rounding cell of the current stored value, i.e.
\begin{equation}
|\Delta_t| < \frac{u_t}{2}.
\label{eq:stall-ulp-condition}
\end{equation}

Let $\eps=2^{-p}$ denote the mantissa-resolution parameter of the storage
format. We use $\eps=2^{-7}$ for BF16, $\eps=2^{-3}$ for E4M3 FP8, and
$\eps=2^{-2}$ for the unsigned E2M2 FP4 format used for the nonnegative
second moment. For a normalized stored value
$v_{t-1}^q=m_t2^{e_t}$, with $m_t\in[1,2)$, its local spacing is
\begin{equation}
u_t=\operatorname{ulp}(v_{t-1}^q)
=\frac{\eps}{m_t}v_{t-1}^q.
\label{eq:local-spacing}
\end{equation}
Thus, the exact relative spacing depends on the coordinate- and
time-dependent mantissa $m_t$. To obtain a single scalar predictor, we replace
$m_t$ by the representative value
$\bar m:=\mathbb{E}[m_t]=1/\ln 2$
under a log-uniform mantissa model, giving
\begin{equation}
u_t\approx\frac{\eps}{\bar m}v_{t-1}^q.
\label{eq:effective-spacing}
\end{equation}
We refer to this substitution as the \emph{effective-spacing approximation}.
It is a representative-spacing approximation, not an exact averaging of the
stalling probability; its derivation is given in
\Cref{app:ulp-geometry}.

Using the local-stationarity approximation
$v_{t-1}\approx v_{t-1}^q\approx\sigma^2$, define
\begin{equation}
z_t:=\frac{g_t^2}{\sigma^2}
\overset{d}{\approx}\chi_1^2.
\label{eq:z-chi}
\end{equation}
Then
\begin{equation}
\Delta_t\approx(1-\beta_2)\sigma^2(z_t-1),
\label{eq:effective-increment}
\end{equation}
and the effective update-to-spacing ratio is
\begin{equation}
\frac{|\Delta_t|}{u_t}
\approx
\frac{(1-\beta_2)\bar m}{\eps}|z_t-1|.
\label{eq:effective-update-spacing}
\end{equation}
\begin{definition}[Effective precision ratio]
\label{def:rhohat}
We define
\begin{equation}
\hat\rho := \frac{\eps}{2(1-\beta_2)\bar m}.
\label{eq:rhohat}
\end{equation}
\end{definition}
\noindent With this convention,
\begin{equation}
\frac{|\Delta_t|}{u_t}\approx \frac{|z_t-1|}{2\hat\rho},
\label{eq:delta-over-u-rho}
\end{equation}
so all stalling statements below are controlled by $\hat\rho$.

\subsection{Nearest rounding induces hard-gated stalling}
\label{sec:nr-stalling}

Under nearest rounding (NR), the condition \eqref{eq:stall-ulp-condition}
becomes, under the effective-spacing approximation,
\begin{equation}
|z_t-1|< \hat\rho.
\label{eq:nr-gate}
\end{equation}
Equivalently, the stored value changes only when $|z_t-1|\ge \hat\rho$.

\begin{proposition}[NR stalling probability]
\label{prop:nr-stall}
Let $z_t\sim \chi_1^2$. Under the effective-spacing approximation, the
one-step probability of state-update stalling under nearest rounding is
\begin{equation}
P^{\mathrm{NR}}_{\mathrm{stall}}(\hat\rho)
\approx
F_{\chi^2_1}(1+\hat\rho)-F_{\chi^2_1}(\max(0,1-\hat\rho)).
\label{eq:nr-stall-prob}
\end{equation}
\end{proposition}

\begin{corollary}[High-stalling regime under nearest rounding]
\label{cor:nr-large-rho}
When $\hat\rho\ge 1$, the lower term in
\eqref{eq:nr-stall-prob} vanishes, so
\begin{equation}
P^{\mathrm{NR}}_{\mathrm{stall}}(\hat\rho)
\approx
F_{\chi^2_1}(1+\hat\rho).
\label{eq:nr-stall-prob-simple}
\end{equation}
Moreover, $F_{\chi^2_1}(1+\hat\rho)\to 1$ as
$\hat\rho\to\infty$. Thus, when $\hat\rho$ is sufficiently large,
the stalling probability is close to one.
Thus, once the EMA reaches its steady magnitude, most updates are blocked by
the rounding gate.
\end{corollary}

For $\beta_2=0.999$, the effective ratios from \eqref{eq:rhohat} give the
steady-state stalling probabilities in \Cref{tab:nr-ss-stall} which are also verified empirically in \Cref{fig:staleness}.

\begin{table}[H]
\small
\centering
\begin{tabular}{lccc}
\toprule
Format & $\eps$ & $\hat\rho$ & $P^{\mathrm{NR}}_{\mathrm{stall}}(\hat\rho)$ \\
\midrule
BF16          & $2^{-7}$ & $2.71$  & $\approx 0.946$ \\
FP8 (E4M3)    & $2^{-3}$ & $43.3$  & $\approx 1.000$ \\
FP4 (E2M2)    & $2^{-2}$ & $86.6$  & $\approx 1.000$ \\
\bottomrule
\end{tabular}
\caption{Theoretical steady-state stalling probabilities under nearest rounding (NR) from \eqref{eq:rhohat}.}
\vspace{-1em}
\label{tab:nr-ss-stall}
\end{table}

These values already reveal the core problem: at FP8 and FP4, once the second
moment reaches its steady-state scale, \emph{almost every step stalls}. In this
regime the EMA rarely changes, so the optimizer effectively stops updating its
estimate of gradient magnitude and instead operates with nearly frozen
statistics. Importantly, this effect is largely independent of model or
training scale. The analysis is expressed in terms of the normalized variable
$z_t = g_t^2 / \mathbb{E}[g^2]$, whose distribution does not depend on the
absolute magnitude of gradients, while the ratio $\hat\rho$ governing stalling
depends only on the precision format and decay $\beta_2$. As a result, changing
the overall gradient scale, model width, or dataset shifts the EMA magnitude
and floating-point spacing together, leaving the stalling probability
essentially unchanged. We jointly validate this NR prediction and its SR counterpart across
architectures and model scales in
\Cref{fig:pss-cross-architecture}.

In addition, we empirically audit the key assumptions underlying these predictions in \Cref{app:assumption-audit}, including the normalized-gradient model, stalling-region mass, mean mantissa, and relative ULP spacing. The close agreement across these diagnostics supports the modeling approximations and helps explain the accuracy of the resulting stalling predictions. We therefore present the theory as a compact predictive model for state-update stalling and reset-scale selection, rather than as an exact distributional characterization.

\subsection{Empirical sensitivity to simulated moment stalling}
\label{sec:simulated-stalling}

To isolate the effect of stale EMA updates, we run a controlled intervention
during LLaMA-60M pre-training on WikiText: after warmup, each update of a chosen
moment buffer is independently skipped with probability $p$. This is not meant
as a literal model of quantization-induced stalling, but as a direct stress
test of optimizer sensitivity to persistent unchanged-state behavior.

\begin{figure}[H]
    \centering
    \includegraphics[width=0.9\linewidth]{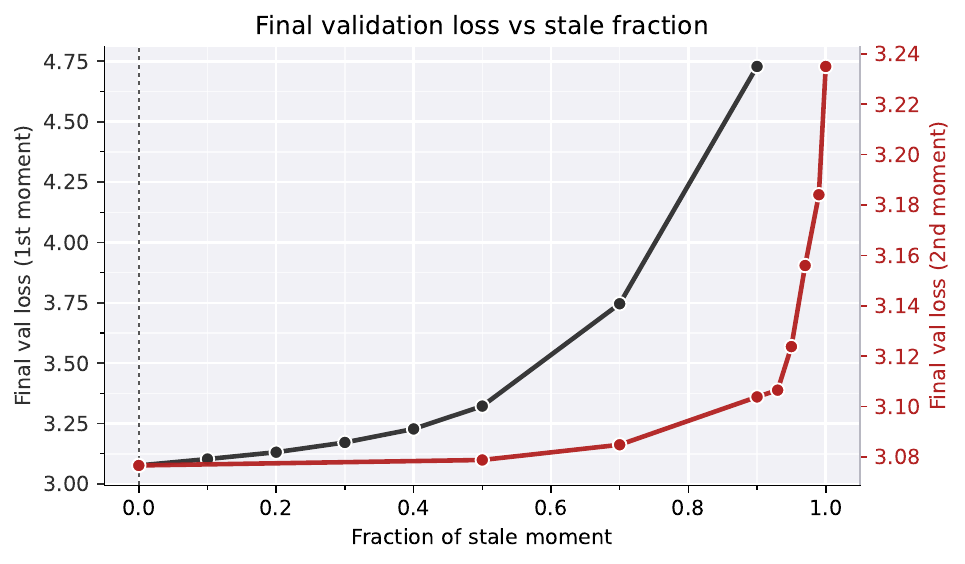}
    \caption{Simulated EMA stalling during LLaMA-60M pre-training on WikiText.
    After warmup, each moment update is skipped independently with probability
    $p$.}
    \label{fig:simulated_staleness}
\end{figure}

\Cref{fig:simulated_staleness} shows a clear trend: second-moment stalling
degrades convergence gradually, mainly through worse final loss, whereas
first-moment stalling is far more destabilizing. This is consistent with their
roles in Adam-like methods: the second moment mainly controls effective
learning-rate scaling, while the first moment carries directional information
and is therefore more sensitive to missed updates.

\section{What stalling does to the EMA dynamics}
\label{sec:stall-consequences}

\subsection{A coarse-grained effective decay approximation}
\label{sec:beta-eff}

A useful summary of stalling is through an \emph{effective decay}
approximation. If a step stalls, the stored state remains unchanged
($v_t^q=v_{t-1}^q$), which is equivalent to using decay $1$ on that step rather
than the nominal decay $\beta_2$. If the empirical stalling rate is
approximately constant at $P_{\mathrm{stall}}$ over a short window, only a
fraction $1-P_{\mathrm{stall}}$ of steps effectively apply the nominal EMA
contraction, giving
\begin{equation}
\beta_2^{\mathrm{eff}} \approx 1-(1-\beta_2)(1-P_{\mathrm{stall}}),
\label{eq:beta-eff-rewrite}
\end{equation}
and, for $\tau:=\frac{1}{1-\beta_2}$,
\begin{equation}
\tau_{\mathrm{eff}} \approx \frac{1}{(1-\beta_2)(1-P_{\mathrm{stall}})}
= \frac{\tau}{1-P_{\mathrm{stall}}}.
\label{eq:tau-eff-rewrite}
\end{equation}
This is not an exact reduction of the quantized recursion, since stalling
events are state dependent and correlated across steps, but a local mean-field
summary of how a high stalled fraction slows EMA response. The same persistence
also makes successive optimizer updates more temporally correlated, a behavior
linked to instability in Adam-like methods \cite{molybog2023theory}.

\subsection{A finite startup window before stalling dominates}
\label{sec:startup-window}

Stalling does not occur immediately after state initialization. When the EMA
state is small, the quantization grid is fine relative to the typical update,
so many updates pass through. As the state grows toward its stationary scale,
the probability of a stalled update rises, creating a finite \emph{startup
window} during which the EMA remains meaningfully responsive.

Under the same idealized model as above, after $j$ updates have been
accumulated from a zero-initialized state, the reference high precision trajectory is
\begin{equation}
v_j^{\mathrm{ref}} := \sigma^2 \phi_j,
\qquad
\phi_j:=1-\beta_2^j.
\label{eq:phi-j}
\end{equation}

\begin{proposition}[Transient NR stalling probability]
\label{prop:transient-stall}
With $z_t$ defined in \eqref{eq:z-chi}, the transient stalling probability
under nearest rounding is approximated by
\begin{equation}
P_{\mathrm{stall}}^{\mathrm{NR}}(j)
\approx
F_{\chi^2_1}\!\bigl((1+\hat\rho)\phi_j\bigr),
\label{eq:pstall-transient-main}
\end{equation}
for the formats of interest, where $\hat\rho\ge 1$.
\end{proposition}

This expression centers the stalling event at the \emph{current} state scale
$\phi_j$ rather than the steady-state value, and therefore captures the
increase in stalling as the EMA grows away from initialization.

\begin{definition}[Startup window length]
\label{def:startup-window}
For a tolerance $P_0\in(0,1)$, define the \emph{startup window length} as the
smallest $j$ such that $P_{\mathrm{stall}}(j)\ge P_0.$
\end{definition}

\begin{proposition}[Ideal startup window]
\label{prop:jstar-ideal}
Let
\[
\phi^\star(P_0):=\frac{F^{-1}_{\chi^2_1}(P_0)}{1+\hat\rho}.
\]
Whenever $\phi^\star(P_0)<1$, the ideal startup window is
\begin{equation}
j^\star_{\mathrm{ideal}}(P_0)
=
\left\lceil
\frac{\log\!\bigl(1-\phi^\star(P_0)\bigr)}{\log\beta_2}
\right\rceil.
\label{eq:jstar-corrected-main}
\end{equation}
\end{proposition}

In practice, measured stalling often starts from a nonzero floor
$P_{\mathrm{init}}$, so we use
\begin{equation}
P_{\mathrm{total}}(j)
=
P_{\mathrm{init}} + (1-P_{\mathrm{init}})\,P_{\mathrm{stall}}(j),
\label{eq:Ptotal-main}
\end{equation}
which is equivalent to replacing $P_0$ by the effective threshold
\begin{equation}
P_0^{\mathrm{eff}}
:=
\frac{P_0-P_{\mathrm{init}}}{1-P_{\mathrm{init}}},
\label{eq:P0eff-main}
\end{equation}
and therefore
\begin{equation}
j^\star(P_0)=
\begin{cases}
0, & P_0\le P_{\mathrm{init}},\\[3pt]
j^\star_{\mathrm{ideal}}(P_0^{\mathrm{eff}}), & P_0>P_{\mathrm{init}}.
\end{cases}
\label{eq:jstar-total}
\end{equation}

\begin{figure*}[t]

    \centering
    \includegraphics[width=\linewidth]{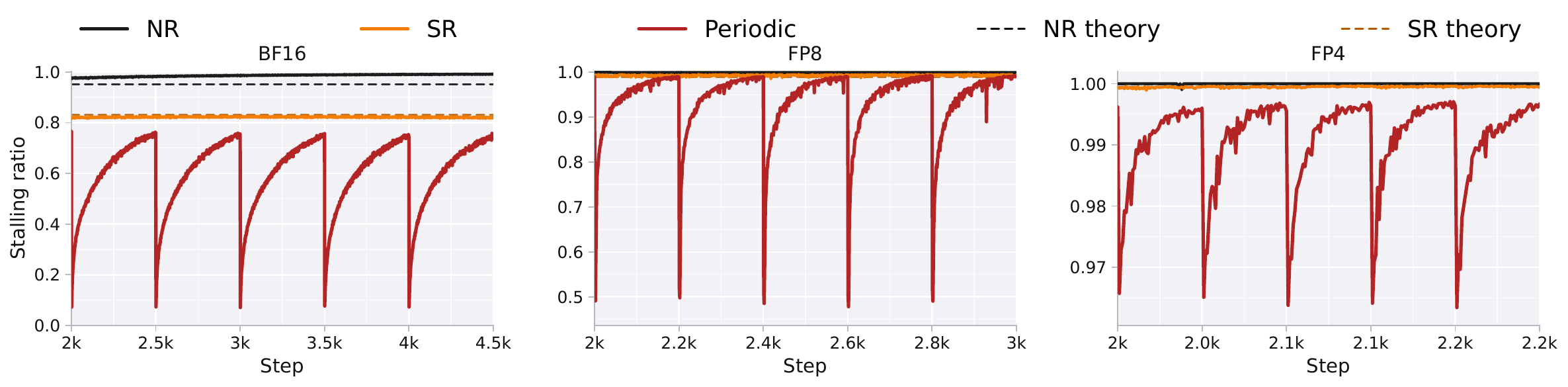}
     \vspace{-2em}
    \caption{Measured stalling fraction during LLM pre-training for quantized
    EMA states. As the second moment grows away from initialization, the
    stalled fraction rises and then plateaus near its steady-state level.
    Lower-precision formats enter the stalled regime earlier and remain there
    longer. 100M-parameter LLaMa model trained on C4.}
    \vspace{-1em}
    \label{fig:staleness}
\end{figure*}

Details on the empirical floor $P_{\mathrm{init}}$ and numerical startup-window
values are deferred to \Cref{app:init-floor,app:jstar}.
\Cref{fig:staleness} illustrates two main trends. First, the stalled fraction
is already nonzero immediately after initialization, and this empirical floor
$P_{\mathrm{init}}$ is itself larger at lower precision; we discuss its origin
and estimate it in \Cref{app:init-floor}. Second, the stalled fraction rises as
the second moment grows toward its steady-state scale, producing a finite
responsive startup window before the EMA becomes largely stale. This window is
longest in BF16 and substantially shorter in FP8; for FP4, accounting for the
empirical floor is essential to separate the reset-sensitive buildup of
stalling from the large format-dependent baseline. This precision-dependent
startup behavior will later provide the key intuition for why state resets can
help and why their timing matters.

\subsection{First moment stalling}
\label{sec:first-moment-restarts}

The first moment behaves qualitatively similar to the second even though its
update scale and ulp scale do not normalize in the same way. For the second
moment, the update is driven by $g_t^2$ and the stored state is of the same
order as the local variance, so the variance scale largely cancels from the
update-to-ulp ratio, yielding the nearly universal approximation developed
above. For the first moment, $
u_t = \beta_1 u_{t-1} + (1-\beta_1)g_t$,
the relevant scale is set directly by the signed gradient. As a result, first-moment stalling depends on the local signal-to-noise ratio
between the mean gradient and its fluctuations, and therefore varies across
coordinates rather than being
determined primarily by the precision format. This makes the first moment less
amenable to a simple closed-form analysis of the kind developed above for the
second moment. Empirically, however, its behavior follows a qualitatively
similar pattern, as we show in the Appendix \Cref{app:first_moment}.

\section{A recipe for stalling mitigation}
\label{sec:solutions-rewrite}

Quantized EMA states fail primarily because stalling makes the effective EMA
dynamics too slow. We discuss two complementary ways to mitigate this.

\subsection{Stochastic rounding}
\label{sec:sr-rewrite}

For nonrepresentable $x$, let $q_-(x)<x<q_+(x)$ be its neighbors.
Stochastic rounding (SR) assigns distance-proportional probabilities
\cite{gupta2015deep,croci2022stochastic,Ozkara2025StochasticRF}. With
$p(x):=(x-q_-(x))/(q_+(x)-q_-(x))$,
\begin{equation}
Q_{\mathrm{SR}}(x)=
\begin{cases}
q_+(x), & \text{with probability }p(x),\\
q_-(x), & \text{with probability }1-p(x).
\end{cases}
\label{eq:sr-definition}
\end{equation}
SR is unbiased without clipping and softens, but does not eliminate, stalling.

\begin{proposition}[SR stalling probability]
\label{prop:sr-stall}
Under the effective-spacing approximation at steady state, with $z_t$ as in
\eqref{eq:z-chi}, the conditional probability that SR returns to the current
stored value is
\begin{equation}
p^{\mathrm{SR}}_{\mathrm{stall}}(z_t;\hat\rho)
=
\max\!\Bigl(0,\;1-\frac{|z_t-1|}{2\hat\rho}\Bigr),
\label{eq:sr-soft-gate}
\end{equation}
and hence $P^{\mathrm{SR}}_{\mathrm{stall}}(\hat\rho)
=
\mathbb{E}\!\left[
p^{\mathrm{SR}}_{\mathrm{stall}}(z_t;\hat\rho)
\right]$
\end{proposition}

SR is only a partial fix. At $\beta_2=0.999$, the steady-state stalling
probabilities remain high (\Cref{tab:sr-stall}) and are consistent with the
empirical behavior in \Cref{fig:staleness}.
\begin{table}[H]
\centering
\small
\begin{tabular}{lccc}
\toprule
Format & BF16 & FP8 & FP4 \\
\midrule
$P^{\mathrm{SR}}_{\mathrm{stall}}$ & $0.825$ & $0.989$ & $0.994$ \\
\bottomrule
\end{tabular}
\caption{Theoretical steady-state stalling probabilities under SR for $\beta_2=0.999$.}
\vspace{-1em}
\label{tab:sr-stall}
\end{table}
Thus SR materially helps at BF16, but at FP8 and FP4 the stalled regime
remains dominant. This analysis isolates unchanged-state probabilities only;
in practice SR may also help by reducing directional quantization bias
\cite{croci2022stochastic}.

\begin{figure}[h]
\hspace*{-0.5cm}
    \centering
    \includegraphics[width=\columnwidth]{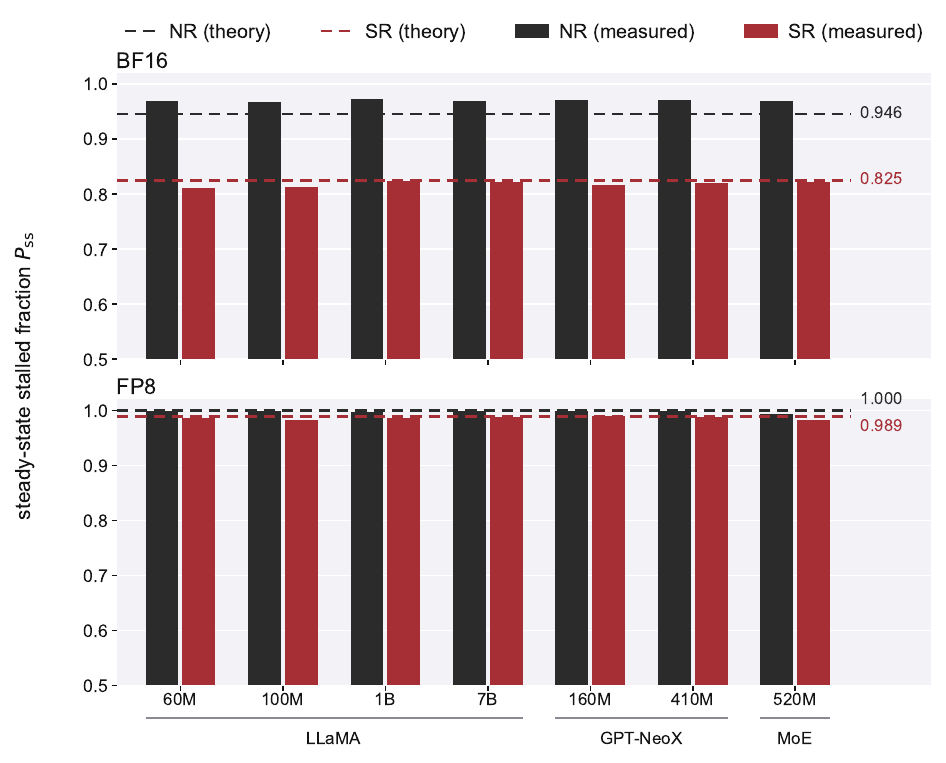}
        \vspace{-1em}
    \caption{
    Agreement between predicted and measured
    steady-state stalled fractions.
    Bars show empirical measurements, while dashed lines show the corresponding theoretical predictions.
    }
     \vspace{-1em}
    \label{fig:pss-cross-architecture}
\end{figure}

\paragraph{Agreement across architectures and scales.}
\Cref{fig:pss-cross-architecture} provides a striking validation of our
theory across model families and scales. Despite spanning LLaMA,
GPT-NeoX, and MoE architectures and more than two orders of magnitude in
model scale, from 60M to 7B parameters, the same theoretical prediction
for each precision/rounding configuration remarkably tracks all measurements. The maximum absolute deviation between theory and measurement is only $0.028$, strongly supporting our central theoretical conclusion
that, for fixed $\beta_2$, the steady-state stalled fraction is governed primarily by numerical precision and rounding mode across the models considered. Complete measurements and details of the large-scale diagnostic
protocol are provided in \Cref{app:cross-architecture-stalling}.

\subsection{Periodic state resetting}
\label{sec:resets-rewrite}

When steady-state stalling is high, a quantized EMA eventually becomes nearly
frozen, so resetting the state is a natural way to restore responsiveness.
The transient analysis above shows, however, that this responsiveness is not
lost immediately: after each initialization there is a finite startup window
during which the EMA remains adaptive before stalling dominates. Since the
length of this window is governed primarily by the effective precision ratio
$\hat\rho$ in \Cref{def:rhohat}, it is largely stable throughout training for a
fixed precision regime. This both motivates periodic resets and suggests how
to time them: resets should occur once the startup window is largely exhausted,
but not so early that they discard useful averaging.

We compare two quantities over a cycle of length $K$: the accumulated
\emph{reset-sensitive} staleness, and the statistical benefit still left to
gain from continuing the EMA. Removing the reset-insensitive floor
$P_{\mathrm{init}}$, we define the normalized stalling progress
\begin{equation}
S(j)
:=
\frac{P_{\mathrm{total}}(j)-P_{\mathrm{init}}}
     {P_{\mathrm{total}}^{\mathrm{ss}}-P_{\mathrm{init}}}
=
\frac{P_{\mathrm{stall}}(j)}{P_{\mathrm{stall}}^{\mathrm{ss}}},
\label{eq:S-of-j-main}
\end{equation}
where $P_{\mathrm{total}}^{\mathrm{ss}}
=
P_{\mathrm{init}} + (1-P_{\mathrm{init}})\,P_{\mathrm{stall}}^{\mathrm{ss}}.
$
Rather than using the endpoint $S(K)$ alone, we average over the cycle and
count only the excess above a tolerance level $s_0$:
\begin{equation}
\bar S_{s_0}(K)
:=
\frac{1}{K}\sum_{j=1}^{K}
\left[
\frac{S(j)-s_0}{1-s_0}
\right]_+,
\label{eq:avgS-main}
\end{equation}
and compare it with the remaining statistical error of the bias-corrected EMA,
whose derivation is given in the Appendix \Cref{app:reset-heuristic},
\begin{equation}
E(K):=\frac{2\beta_2^K}{1+\beta_2^K}.
\label{eq:E-of-K}
\end{equation}

\begin{proposition}[Reset period]
\label{prop:restart-heuristic}
Define
\begin{equation}
K^\star
:=
\min\bigl\{
K\ge 1:\ \bar S_{s_0}(K)\ge E(K)
\bigr\}.
\label{eq:reset-heuristic-main}
\end{equation}
Using $P_{\mathrm{stall}}(j)
\approx
F_{\chi^2_1}\!\bigl((1+\hat\rho)(1-\beta_2^j)\bigr)$
and $P_{\mathrm{stall}}^{\mathrm{ss}}
\approx
F_{\chi^2_1}(1+\hat\rho)$
the condition in \eqref{eq:reset-heuristic-main} is a one-dimensional monotone
test in $K$. We use $s_0=0.6$ throughout.
\end{proposition}

\begin{figure*}[t]
\centering
 \vspace{-1em}
\includegraphics[width=0.9\linewidth]{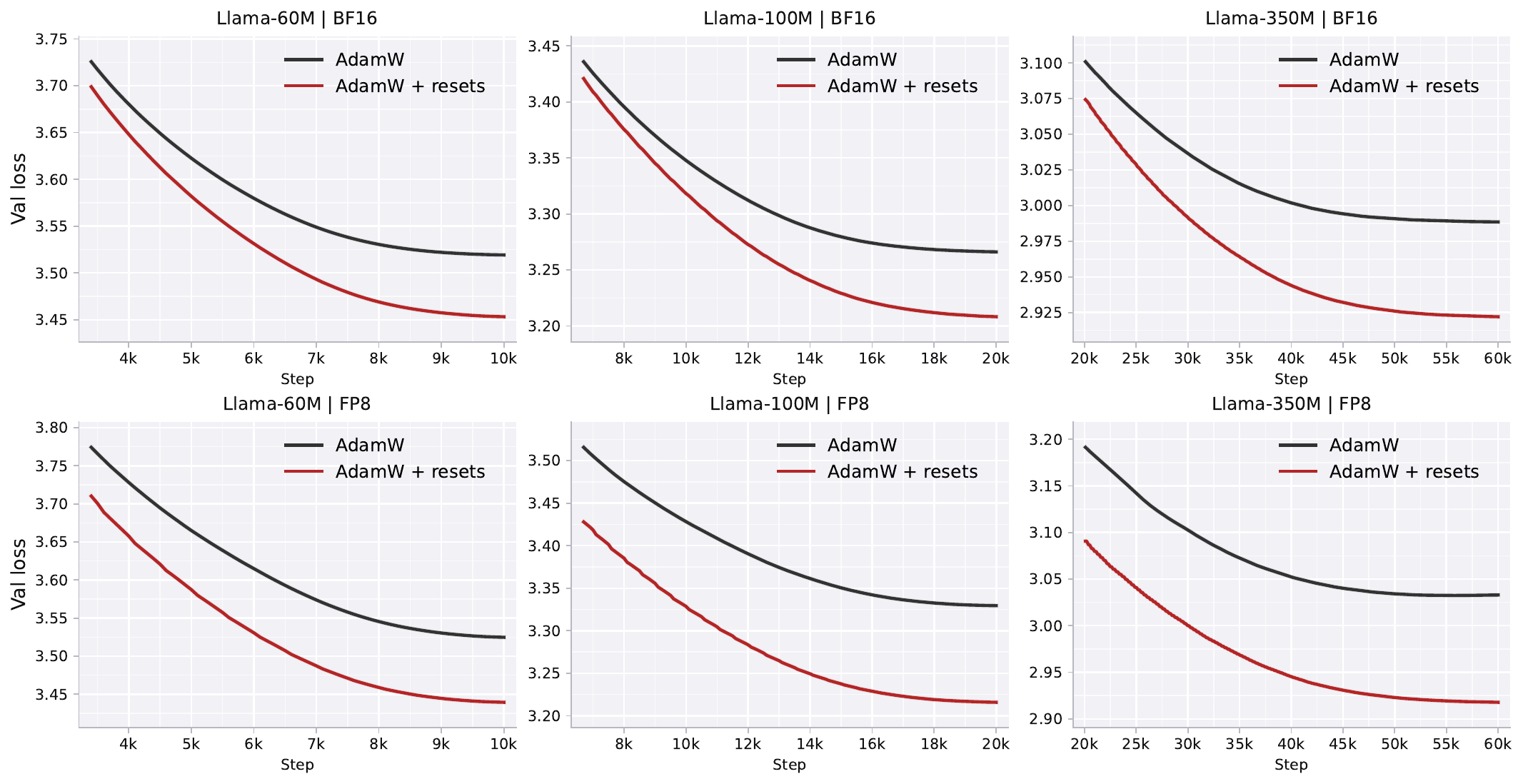}
 \vspace{-1em}
\caption{Validation loss curves for the 3 model scales under 2 different precisions}
 \vspace{-1em}
\label{fig:training-curves}
\end{figure*}

Intuitively, we reset once the cycle spends enough time in a highly stalled
regime that the remaining benefit of further averaging is no longer worth the
cost. Solving \eqref{eq:reset-heuristic-main} for the precisions considered in
this work gives the periods in \Cref{tab:restart-periods-table}, which should be
viewed as selecting a good operating region rather than a uniquely optimal
period.

\begin{table}[H]
\centering
\small
\begin{tabular}{lccc}
\toprule
 & BF16 & FP8 & FP4 \\
\midrule
$K^\star$ & $\approx 1116$ & $\approx 320$ & $\approx 224$ \\
\bottomrule
\end{tabular}
\caption{Derived reset periods from \eqref{eq:reset-heuristic-main} with
$\beta_2=0.999$ and $s_0=0.6$.}
\vspace{-1em}
\label{tab:restart-periods-table}
\end{table}

Resetting the first moment is mainly useful in very aggressive regimes such as
FP4, since at higher precisions its steady-state stalling rate is often below
the initial post-reset level $P_{\mathrm{init}}$, as also shown in
\Cref{app:first_moment}.

\section{Experiments}

\subsection{Settings}
\label{sec:settings}

We follow the C4 pre-training protocols of
\citet{pmlr-v235-zhao24s} and \citet{lialin2024relora}, using the LLaMA
architecture family~\cite{llama1} and matching the small-scale GaLore setup
for the 60M, 100M, and 350M models. Concretely, the models are trained for 10K/20K/60K steps on
approximately 1.3B/2.6B/7.8B tokens, with maximum sequence length 256, token
batch size 131K, linear warmup over the first 10\% of training, cosine decay
to 10\% of the initial learning rate, and no data repetition. We evaluate
three optimizer-state precision regimes: BF16, FP8 using E4M3 for both
moments, and FP4 with the first moment in FP4 (E2M1) and the unsigned second
moment in FP4 (E2M2). For FP8 we use per-tensor scaling; for FP4 we use
block-wise scaling with block size 128 and exclude the zero point from the
second-moment grid, following \citet{4bit_states}. Hyperparameters are given in \Cref{app:training-details}.

For each regime, we compare vanilla AdamW against AdamW with stochastic
rounding and periodic EMA resets. Unless otherwise stated, the second-moment reset period is chosen using
\Cref{eq:reset-heuristic-main}, with $P_{\mathrm{init}}$ estimated
empirically from the stalled fraction at the first step after state
initialization and $s_0=0.6$. Although this rule is derived under nearest
rounding, we use it as a practical proxy under stochastic rounding as well:
SR lowers unchanged-state probabilities but preserves the same qualitative
dependence on precision and decay, so the resulting periods still identify a
good reset regime. This gives periods of $1000$ for BF16, $300$ for FP8, and
$200$ for FP4. The first moment is reset with the same period.

\subsection{Results}

We first evaluate whether periodic EMA resets improve training under quantized
optimizer-state storage across several practical low-precision regimes.
\Cref{tab:main-results} and \Cref{fig:training-curves} show that stochastic
rounding and periodic EMA resets consistently improve final performance in
BF16, FP8, and FP4. Although implementation choices such as scaling affect how
severe stalling becomes, they do not remove the underlying problem of staleness. 

\begin{table}[H]
\small
\centering
\begin{tabular}{lccc}
\toprule
Configuration & 60M & 100M & 350M \\
\midrule
BF16 AdamW & 3.52 & 3.27 & 2.99 \\
\rowcolor{gray!15}
BF16 AdamW + SR + Resets & 3.45 & 3.21 & 2.92 \\
FP8 AdamW & 3.53 & 3.33 & 3.02 \\
\rowcolor{gray!15}
FP8 AdamW + SR + Resets & 3.44 & 3.22 & 2.92 \\
FP4 AdamW & 3.52 & 3.32 & 3.03 \\
\rowcolor{gray!15}
FP4 AdamW + SR + Resets & 3.48 & 3.26 & 2.97 \\
\bottomrule
\end{tabular}
\caption{Validation loss across model scales.}
\vspace{-1em}
\label{tab:main-results}
\end{table}

\subsection{Ablation Studies}

All ablations are conducted on the 100M model.

\paragraph{Sensitivity to the Reset Period.}
We vary the second-moment reset period over a wide range while keeping all
other hyperparameters fixed. \Cref{fig:period-ablation} shows that performance
is robust over a broad range, with the best results attained near the period
predicted by \Cref{eq:reset-heuristic-main}. This supports the view that the
useful reset window is governed by quantized EMA stalling.

\begin{figure}[H]
\centering
\hspace{-0.4cm}
\includegraphics[width=0.9\linewidth]{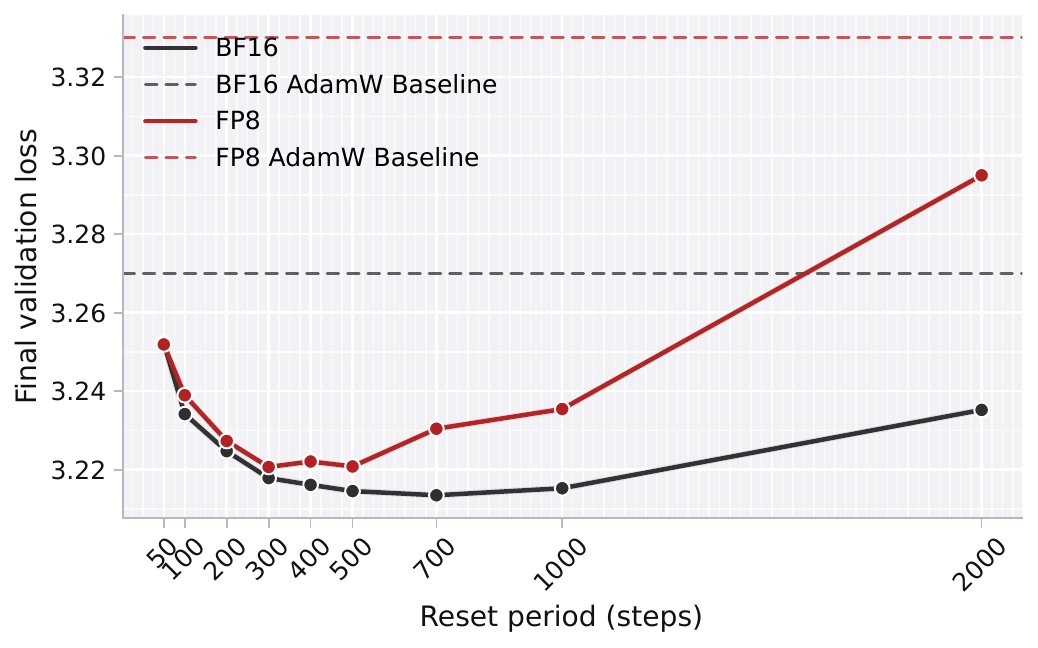}
\vspace{-1em}
\caption{Sensitivity to the reset period.}
\vspace{-1em}
\label{fig:period-ablation}
\end{figure}

\paragraph{Contribution of Different Components.}
We further isolate each mitigation component by keeping training otherwise in
FP32 and quantizing only the optimizer states. \Cref{tab:component-ablation}
shows that both stochastic rounding and resets help on their own in quantized
regimes, while their combination gives the best overall performance. By
contrast, applying the same reset schedule in full FP32 slightly hurts
performance, reinforcing our main claim that the benefit of resets is tied to
low-precision state staleness rather than to resets being universally useful.

\begin{table}[H]
\centering
\small
\begin{tabular}{lcccc}
\toprule
Configuration & FP32 & BF16 & FP8 & FP4 \\
\midrule
NR only & 3.155 & 3.181 & 3.193 & 3.199 \\
SR only & --    & 3.156 & 3.160 & 3.182 \\
Resets only & 3.161 & 3.158 & 3.159 & 3.166 \\
\rowcolor{gray!15}
SR + Resets & -- & 3.155 & 3.155 & 3.160 \\
\midrule
Memory savings & 0\% & 50\% & 75\% & 87.5\% \\
\bottomrule
\end{tabular}
\caption{Component ablation with FP32 training and quantized optimizer states
only. In the FP32 column, we apply resets only, with period 1000 for both
moments.}
\vspace{-1em}
\label{tab:component-ablation}
\end{table}

\paragraph{Asymmetric and Adaptive Resets.}
We next consider two extensions of the basic reset strategy: asymmetric reset
periods for the two moments, and an adaptive rule based on \Cref{prop:restart-heuristic} and detailed in
\Cref{sec:adaptive-restarts}. In the latter case, we set
$P_{\mathrm{stall}}^{\mathrm{ss}}=1$ and $s_0=0.6$, estimate
$P_{\mathrm{stall}}(j)$ online, and trigger a reset once
$\bar S_{s_0}(K)-E(K)>0$. As shown in \Cref{tab:adaptive-ablation}, using no
resets for the first moment and period 300 for the second slightly improves on
the symmetric baseline, consistent with the fact that in FP8 first-moment
steady-state stalling is typically the same as its initial floor
$P_{\mathrm{init}}$, see \Cref{app:first_moment}. 

The adaptive rule matches, but does not improve upon, the tuned periodic schedule. With fixed $\beta_2$, precision, and rounding mode, stalling follows approximately the same post-reset trajectory in every cycle. The adaptive
threshold is therefore reached at nearly the same reset age each time, making the resulting schedule effectively periodic.

\begin{table}[H]
\small
\centering
\begin{tabular}{lc}
\toprule
Reset Strategy & Validation Loss \\
\midrule
Symmetric Resets & 3.223 \\
\rowcolor{gray!15}
Asymmetric Resets & 3.212 \\
Adaptive Resets & 3.225 \\
\bottomrule
\end{tabular}
\caption{Different resetting policies under FP8.}
\vspace{-1em}
\label{tab:adaptive-ablation}
\end{table}

\paragraph{Extensions to Other EMA-Based Optimizers.}
Finally, we test the same reset strategy on SOAP, also resetting the EMA states of its left and right preconditioners using the same period as in AdamW and following the setup in \citet{wang20244}. \Cref{tab:soap-results} shows that SR and resets again improve low-precision performance.

\begin{table}[H]
\small
\centering
\begin{tabular}{lcc}
\toprule
Configuration & BF16 & FP8 \\
\midrule
SOAP & 3.22 & 4.15 \\
\rowcolor{gray!15}
SOAP + SR + Resets & 3.19 & 3.23 \\
\bottomrule
\end{tabular}
\caption{SOAP with and without SR + resets.}
\vspace{-1em}
\label{tab:soap-results}
\end{table}

Additional experiments in \Cref{app:additional-results} include end-to-end pretraining at the 1B scale, longer-horizon training, zero-shot downstream evaluation, comparisons with low-memory optimizers, runtime and memory measurements,
low-precision-weight diagnostics, and an ablation of the quantization block scaling granularity.

\section{Conclusion}

Low-precision optimizer states do not just save memory; they change the
dynamics of optimization. Our central message is that quantized EMA states can
become effectively stale for long stretches of training, and that this
staleness is a useful lens for understanding both the failure modes of
low-precision training and the surprising effectiveness of state resets.
Viewed this way, resets are not merely a heuristic borrowed from successful
training recipes: they are a targeted way to restore responsiveness once the
optimizer state has stopped meaningfully evolving. Just as importantly, our
analysis suggests that good resets are not arbitrary, their timing matters.
We hope this perspective helps shift the discussion of low-precision optimizer
states from storage alone to dynamics, and motivates future work on optimizers
and quantization schemes designed explicitly to remain responsive under
quantization.

\newpage
\clearpage
\section{Limitations}

Our analysis is designed to isolate a specific and practically important mechanism, state-update stalling under quantized EMA storage, rather than to model every aspect of low-precision training. As a result, some components are intentionally approximate, including the local scalar analysis of the EMA dynamics and the rule used to choose reset periods. These simplifications still match the main empirical trends well in our experiments, but the exact quantitative behavior can depend on implementation choices such as scaling strategy, rounding mode, and optimizer variant. In addition, our end-to-end study focuses on the LLaMA pre-training settings and EMA-based optimizers considered here. Extending the analysis to broader training regimes, larger scales, and additional classes of stateful optimizers is a natural direction for future work.
\bibliography{custom}

\appendix
\twocolumn
\section{Appendix}

\subsection{Floating-point grid spacing and the effective precision ratio}
\label{app:ulp-geometry}

For a normalized floating-point number,
\[
x = m\,2^e,
\qquad m\in[1,2),\quad e\in\mathbb{Z},
\]
the spacing between adjacent representable numbers within a fixed exponent
binade is
\[
\operatorname{ulp}(x)=2^{e-p},
\]
where $p$ is the number of stored mantissa bits. Hence
\begin{equation}
\operatorname{ulp}(x)=2^{e-p}=\frac{2^{-p}}{m}\,x,
\label{eq:ulp}
\end{equation}
so the exact local relative spacing is
\[
\frac{\operatorname{ulp}(x)}{x}=\frac{2^{-p}}{m}.
\]

In the main text we do not carry this binade- and mantissa-dependent quantity
through all formulas. Instead, we summarize each format by a single parameter
$\eps$, writing
\[
\operatorname{ulp}(x)\approx \frac{\eps}{m}\,x.
\]
For the normalized regimes considered here, we take $\eps=2^{-p}$; in
particular, $\eps=2^{-7}$ for BF16, $\eps=2^{-3}$ for E4M3 FP8, and
$\eps=2^{-2}$ for the unsigned E2M2 FP4 format used for the second moment.

In our setting, the current EMA state is $v_{t-1}$ and the high precision proposal is
\[
v_t=\beta_2 v_{t-1}+(1-\beta_2)g_t^2.
\]
Hence the update increment is
\[
\Delta_t:=v_t-v_{t-1}
=(1-\beta_2)(g_t^2-v_{t-1}).
\]
Writing $v_{t-1}=m_t2^{e_t}$ and using \eqref{eq:ulp}, the local spacing around
the stored value is approximated by
\[
u_t=\operatorname{ulp}(v_{t-1}^q)\approx \operatorname{ulp}(v_{t-1})
=\frac{\eps}{m_t}v_{t-1}.
\]
Under nearest rounding, state-update stalling occurs whenever
$|\Delta_t|<u_t/2$; equivalently,
\[
|\Delta_t|\ge \frac{u_t}{2}
\]
is required for the stored value to change. Therefore
\begin{equation}
\frac{|\Delta_t|}{u_t}
=
\frac{(1-\beta_2)|g_t^2-v_{t-1}|}{(\eps/m_t)\,v_{t-1}}.
\label{eq:delta-over-u-app-pre}
\end{equation}

Under the local-stationarity approximation $v_{t-1}\approx \sigma^2$, and with
\[
z_t:=\frac{g_t^2}{\sigma^2}\sim\chi_1^2,
\]
this becomes
\begin{equation}
\frac{|\Delta_t|}{u_t}
\approx
\frac{(1-\beta_2)m_t|z_t-1|}{\eps}.
\label{eq:delta-over-u-app}
\end{equation}
Hence the nearest-rounding gate $|\Delta_t|\ge u_t/2$ is equivalent to
\begin{equation}
|z_t-1|\ge \frac{\eps}{2(1-\beta_2)m_t}.
\label{eq:threshold-m}
\end{equation}

The threshold depends on the normalized mantissa $m_t$, which varies across
coordinates. To reduce this dependence to a single scalar, we adopt a
log-uniform mantissa model and treat $\log_2 m_t$ as approximately uniform on
$[0,1]$, corresponding to
\[
f(m)=\frac{1}{m\ln 2},
\qquad m\in[1,2].
\]
Under this model,
\begin{equation}
\bar m:=\mathbb{E}[m]
=\int_1^2 m\cdot\frac{1}{m\ln 2}\,dm
=\frac{1}{\ln 2}.
\label{eq:mbar}
\end{equation}
We then replace the coordinate-dependent threshold in \eqref{eq:threshold-m} by
its single-scalar surrogate obtained from $\bar m$. This is an
effective-spacing approximation rather than an exact averaging of the stalling
probability, and it is used only to obtain the compact predictor $\hat\rho$:
\begin{equation}
\hat\rho:=\frac{\eps}{2(1-\beta_2)\bar m},
\qquad
\bar m=\frac{1}{\ln 2}.
\label{eq:rhohat-app}
\end{equation}
Thus the gate is approximated by $|z_t-1|\ge \hat\rho$, or equivalently
\[
\frac{|\Delta_t|}{u_t}\approx \frac{|z_t-1|}{2\hat\rho}.
\]

\subsection{NR stalling probability}
\label{app:nr}

We derive \Cref{prop:nr-stall}. Under nearest rounding (NR), the stored state
stalls if and only if the high precision proposal rounds back to the current stored
value. Under the hard-gate approximation from the main text,
\[
|\Delta_t| < \frac{u_t}{2}
\quad\Longleftrightarrow\quad
|z_t-1|<\hat\rho,
\qquad z_t\sim\chi^2_1.
\]
Equivalently, the stored value changes only when $|z_t-1|\ge \hat\rho$.
Since $z_t\ge 0$, the stalling region is
\[
z_t \in \bigl(\max(0,\,1-\hat\rho),\,1+\hat\rho\bigr).
\]
Therefore,
\begin{align}
P^{\mathrm{NR}}_{\mathrm{stall}}(\hat\rho)
&=
\mathbb{P}\!\left(\max(0,1-\hat\rho) < z_t < 1+\hat\rho\right)\\
&=
F_{\chi^2_1}(1+\hat\rho)-F_{\chi^2_1}(\max(0,1-\hat\rho)).
\label{eq:nr-stall-app}
\end{align}

\paragraph{Useful closed form for the $\chi^2_1$ CDF.}
If $z=X^2$ with $X\sim\mathcal{N}(0,1)$, then for $x\ge 0$,
\[
F_{\chi^2_1}(x)=\mathbb{P}(|X|\le \sqrt{x})
=2\Phi(\sqrt{x})-1.
\]
This identity is convenient for numerical evaluation and inversion of the
stalling thresholds.
\subsection{SR stalling probability}
\label{app:sr}

We derive \Cref{prop:sr-stall}. Fix a grid with spacing $u$ around the current
stored value $\vbar$, and suppose the higher-precision proposal is
$\vtil=\vbar+\delta$. Under stochastic rounding (SR), if $\vtil$ lies between
$\vbar$ and an adjacent grid point at distance $u$, then the probability of
rounding back to $\vbar$ is proportional to proximity:
\[
\mathbb{P}\!\bigl(\Q_{\mathrm{SR}}(\vtil)=\vbar \mid \delta\bigr)
=
\begin{cases}
1-|\delta|/u, & |\delta|<u,\\
0, & |\delta|\ge u.
\end{cases}
\]
At steady state, writing $\alpha:=1-\beta_2$, using
\[
\delta \approx \alpha \vbar(z_t-1),
\qquad
u\approx (\eps/\bar m)\vbar,
\qquad
z_t\sim\chi^2_1,
\]
gives
\[
\frac{|\delta|}{u}
\approx
\frac{\alpha \bar m}{\eps}|z_t-1|
=
\frac{|z_t-1|}{2\hat\rho}.
\]
Hence the conditional SR stalling probability is
\begin{equation}
p^{\mathrm{SR}}_{\mathrm{stall}}(z_t;\hat\rho)
=
\max\!\Bigl(0,\;1-\frac{|z_t-1|}{2\hat\rho}\Bigr),
\label{eq:sr-stall-cond-app}
\end{equation}
and averaging over $z_t$ yields
\begin{equation}
P^{\mathrm{SR}}_{\mathrm{stall}}(\hat\rho)
=
\mathbb{E}_{z_t\sim\chi^2_1}
\left[
\max\!\Bigl(0,\;1-\frac{|z_t-1|}{2\hat\rho}\Bigr)
\right].
\label{eq:sr-stall-app}
\end{equation}

\paragraph{Large-$\hat\rho$ approximation.}
When $\hat\rho$ is large, the truncation at $0$ is rarely active, since
$|z_t-1|$ is typically $O(1)$. Dropping the truncation gives
\[
P^{\mathrm{SR}}_{\mathrm{stall}}(\hat\rho)
\approx
1-\frac{1}{2\hat\rho}\,\mu_1,
\qquad
\mu_1 := \mathbb{E}[|z_t-1|].
\]
For $z_t=X^2$ with $X\sim\mathcal{N}(0,1)$, $\mu_1$ has a closed form.

\begin{lemma}[Closed form for $\mu_1$]
\label{lem:mu1}
If $X\sim\mathcal{N}(0,1)$ and $z_t=X^2$, then
\[
\mu_1 = \mathbb{E}[|X^2-1|] = 4\phi(1)
= \frac{4}{\sqrt{2\pi}}e^{-1/2}
\approx 0.9679,
\]
where $\phi$ is the standard normal density.
\end{lemma}

\begin{proof}
By symmetry,
\begin{align*}
\mu_1 &= \mathbb{E}[|X^2-1|]
=2\int_0^\infty |x^2-1|\phi(x)\,dx\\
&=2\!\left(\int_0^1 (1-x^2)\phi(x)\,dx+\int_1^\infty (x^2-1)\phi(x)\,dx\right).    
\end{align*}

Since $\mathbb{E}[X^2-1]=0$, the positive and negative parts have equal mass:
\[
\int_1^\infty (x^2-1)\phi(x)\,dx
=
\int_0^1 (1-x^2)\phi(x)\,dx.
\]
Hence
\[
\mu_1 = 4\int_0^1 (1-x^2)\phi(x)\,dx.
\]
Using $\phi'(x)=-x\phi(x)$,
\begin{align*}
\int_0^1 x^2\phi(x)\,dx
&=
\bigl[-x\phi(x)\bigr]_0^1+\int_0^1 \phi(x)\,dx\\
&=
-\phi(1)+\int_0^1\phi(x)\,dx,    
\end{align*}
so
\[
\int_0^1 (1-x^2)\phi(x)\,dx=\phi(1),
\]
which gives $\mu_1=4\phi(1)$.
\end{proof}

This subsection isolates the unchanged-state probability under SR. It does not
attempt to capture the full benefit of SR, which may also arise from the
conditional unbiasedness of the quantization error even on steps where the
stored state changes.

\subsection{Complete cross-architecture and scale validation}
\label{app:cross-architecture-stalling}

We report the complete measurements underlying
\Cref{fig:pss-cross-architecture}. All experiments use
$\beta_2=0.999$. The larger-scale evaluations are targeted mechanism diagnostics rather
than complete pretraining runs. In particular, we do not train the 1B
and 7B LLaMA models to a full pretraining token budget. The theoretical
claim concerns the local dynamics of the quantized EMA state, and the
steady-state stalled fraction $P_{\mathrm{ss}}$ reaches its plateau on
a much shorter timescale than end-to-end pretraining. We therefore run
these configurations for a few thousand optimizer steps, which is
sufficient to measure the plateau value of $P_{\mathrm{ss}}$.

\begin{table}[H]
\centering
\small
\setlength{\tabcolsep}{3.5pt}
\begin{tabular}{llccrr}
\toprule
Model & Size & Format & Round. &
Meas. $P_{\mathrm{ss}}$ & Pred. $P_{\mathrm{ss}}$ \\
\midrule
LLaMA & 60M  & BF16 & NR & .969 & .946 \\
LLaMA & 60M  & BF16 & SR & .812 & .825 \\
LLaMA & 60M  & FP8  & NR & .999 & 1.000 \\
LLaMA & 60M  & FP8  & SR & .986 & .989 \\
\midrule
LLaMA & 100M & BF16 & NR & .967 & .946 \\
LLaMA & 100M & BF16 & SR & .814 & .825 \\
LLaMA & 100M & FP8  & NR & .998 & 1.000 \\
LLaMA & 100M & FP8  & SR & .983 & .989 \\
\midrule
LLaMA & 1B   & BF16 & NR & .974 & .946 \\
LLaMA & 1B   & BF16 & SR & .825 & .825 \\
LLaMA & 1B   & FP8  & NR & .997 & 1.000 \\
LLaMA & 1B   & FP8  & SR & .986 & .989 \\
\midrule
LLaMA & 7B   & BF16 & NR & .970 & .946 \\
LLaMA & 7B   & BF16 & SR & .823 & .825 \\
LLaMA & 7B   & FP8  & NR & .998 & 1.000 \\
LLaMA & 7B   & FP8  & SR & .988 & .989 \\
\bottomrule
\end{tabular}
\caption{
Complete LLaMA steady-state stalling measurements at
$\beta_2=0.999$. The same prediction is used at every model scale.
The 1B and 7B results are obtained from targeted diagnostic runs rather
than complete pretraining runs.
}
\label{tab:complete-pss-llama}
\end{table}

\begin{table}[H]
\centering
\small
\setlength{\tabcolsep}{3.5pt}
\begin{tabular}{llccrr}
\toprule
Model & Size & Format & Round. &
Meas. $P_{\mathrm{ss}}$ & Pred. $P_{\mathrm{ss}}$ \\
\midrule
GPT-NeoX & 160M & BF16 & NR & .971 & .946 \\
GPT-NeoX & 160M & BF16 & SR & .817 & .825 \\
GPT-NeoX & 160M & FP8  & NR & .999 & 1.000 \\
GPT-NeoX & 160M & FP8  & SR & .989 & .989 \\
\midrule
GPT-NeoX & 410M & BF16 & NR & .971 & .946 \\
GPT-NeoX & 410M & BF16 & SR & .820 & .825 \\
GPT-NeoX & 410M & FP8  & NR & .999 & 1.000 \\
GPT-NeoX & 410M & FP8  & SR & .988 & .989 \\
\midrule
MoE & 520M & BF16 & NR & .970 & .946 \\
MoE & 520M & BF16 & SR & .822 & .825 \\
MoE & 520M & FP8  & NR & .993 & 1.000 \\
MoE & 520M & FP8  & SR & .982 & .989 \\
\bottomrule
\end{tabular}
\caption{
Complete steady-state stalling measurements for the alternative
architectures.
}
\label{tab:complete-pss-other}
\end{table}

Across the 28 model--precision--rounding configurations in
\Cref{tab:complete-pss-llama,tab:complete-pss-other}, the predictions
hold remarkably well: the maximum absolute deviation from measurement
is only $0.028$. This agreement is striking given that the experiments
span three architecture families and more than two orders of magnitude
in model scale. It provides strong evidence that, at fixed $\beta_2$,
the steady-state stalled fraction is governed primarily by numerical
precision and rounding mode across the models considered.

\subsection{Numerical audit of the modeling assumptions}
\label{app:assumption-audit}

The predictor combines several approximations, especially the
$\chi^2_1$ model for normalized squared gradients and the replacement of
coordinate-dependent mantissas and spacings by a single effective spacing.
We audit these assumptions using post-warmup statistics from 3,000 iterations
of LLaMA-100M training under nearest rounding with BF16, FP8, and FP4
optimizer-state storage.

We evaluate the Kolmogorov--Smirnov (KS) distance between the empirical
normalized squared-gradient distribution and $\chi^2_1$, the absolute error in
the probability mass inside the NR stalling region, the empirical mean
mantissa, and the relative error of the effective ULP-spacing approximation.

\begin{table}[t]
\centering
\small
\setlength{\tabcolsep}{4pt}
\begin{tabular}{@{}lrrrr@{}}
\toprule
Format
& $\chi^2_1$ KS
& \shortstack{Stall-mass\\error}
& \shortstack{Mean\\mantissa}
& \shortstack{Relative ULP\\error} \\
\midrule
BF16 & 0.016 & 0.0001 & 1.318 & 0.193 \\
FP8  & 0.011 & 0.0006 & 1.413 & 0.213 \\
FP4  & 0.018 & 0.0031 & 1.427 & 0.307 \\
\bottomrule
\end{tabular}
\caption{Assumption audit after warmup on LLaMA-100M. The KS,
stall-mass-error, and relative-ULP-error entries are medians across tensors.
The theoretical mean mantissa under the log-uniform approximation is
$1/\ln 2\simeq1.443$.}
\label{tab:assumption-audit}
\end{table}

The diagnostics are accurate, the normalized
squared-gradient distribution is close to $\chi^2_1$, and the median error in
the probability mass inside the stalling region is very small. The empirical
mantissas are also consistent with the log-uniform approximation, while the
relative ULP errors show that the effective spacing is a useful aggregate
approximation rather than an exact coordinate-wise description. Together,
these diagnostics help explain why the resulting stalling predictions remain
accurate across architectures, model scales, and precision formats.

\subsection{Empirical initial floor $P_{\mathrm{init}}$}
\label{app:init-floor}

Empirically, the stalled fraction immediately after initialization is already
nonzero. In the main text we therefore treat $P_{\mathrm{init}}$ as an empirical
correction; here we explain why such a floor is expected and why it depends only
weakly on model scale.

At the first update after a zero-initialized second moment, the higher-precision
proposal is
\[
\tilde v_1 = \alpha g_1^2,
\qquad \alpha := 1-\beta_2.
\]
A coordinate remains at its initialized stored value in two main ways:
(i) the corresponding gradient entry is exactly zero, and
(ii) the proposal is nonzero in higher precision but is mapped back to the
initialized quantized value by low-precision storage. Let
\[
p_{\mathrm{zero}} := \Pr(g_{1,i}=0)
\]
denote the fraction of coordinates with exactly zero gradient at the first
step. This contribution is architecture- and data-dependent, but largely
format-independent.

The second contribution depends on the quantization scheme. Consider a scaling
group of size $B$ (the full tensor under per-tensor scaling, or one block under
block-wise scaling). If the maximum proposal magnitude in the group is mapped to
the format maximum $x_{\max}$, then coordinate $i$ is represented by
\[
x_i^{\mathrm{work}}
=
\frac{g_i^2}{\max_{j\le B} g_j^2}\,x_{\max}.
\]
Under nearest rounding, a nonzero coordinate is mapped back to the initialized
stored value whenever
\begin{equation}
\frac{g_i^2}{\max_{j\le B} g_j^2} < \tau,
\qquad
\tau := \frac{s_{\min}}{2x_{\max}},
\label{eq:tau-crush}
\end{equation}
where $s_{\min}$ is the smallest positive representable value of the format.
For stochastic rounding, the hard threshold is replaced by the corresponding
soft gate, so the probability of returning to the initialized value is smaller
but controlled by the same scale ratio.

To obtain a simple approximation, assume that within one scaling group
\[
\frac{g_i^2}{\sigma^2}\overset{d}{\approx}\chi_1^2.
\]
If the group maximum is replaced by a typical upper order statistic
\[
M_B := F_{\chi^2_1}^{-1}\!\left(1-\frac{1}{B}\right),
\]
then the fraction of nonzero coordinates crushed back to the initialized value
under nearest rounding is approximately
\begin{equation}
f_{\mathrm{crush}}^{\mathrm{NR}}(\tau,B)
\approx
F_{\chi^2_1}\!\bigl(\tau M_B\bigr).
\label{eq:f-crush}
\end{equation}
This is an approximation: it replaces the random group maximum by a
typical high quantile. For stochastic rounding, the analogous crush fraction is
obtained by averaging the soft gate over the same distribution, and is
therefore no larger than the NR value.

Combining exact zeros with quantization-induced crushing gives
\begin{equation}
P_{\mathrm{init}}
\approx
p_{\mathrm{zero}} + (1-p_{\mathrm{zero}})\,f_{\mathrm{crush}}(\tau,B).
\label{eq:Pinit}
\end{equation}
This explains two qualitative features seen in practice. First, BF16 typically
has little additional crushing beyond structural zeros, whereas narrower-range
formats can exhibit a much larger initial floor. Second, the dependence on
model size is weak: under per-tensor scaling it enters only through the typical
order statistic $M_B$, which grows only logarithmically in $B$, while under
fixed block-wise scaling it is governed primarily by the block size itself.

Because subsequent stalling slows the growth of the EMA, the ideal predictor
$j^\star_{\mathrm{ideal}}$ can still rise somewhat too quickly after
initialization. Rather than introducing a fully self-consistent transient
model, in the main text we account for the dominant effect through the
phenomenological correction
\[
P_{\mathrm{total}}(j)
=
P_{\mathrm{init}} + (1-P_{\mathrm{init}})\,P_{\mathrm{stall}}(j),
\]
which preserves the closed-form startup-window expression while incorporating
the empirical post-initialization floor.

\subsection{Deriving the startup window length $j^\star(P_0)$}
\label{app:jstar}

We derive the transient stalling law under an \emph{ideal-trajectory
approximation}. After initialization (or reset) to zero, the unquantized EMA
evolves as
\[
v_j = \beta_2 v_{j-1} + \alpha g_j^2,
\qquad
v_0=0,
\qquad
\alpha:=1-\beta_2.
\]
Over the short window relevant to the startup analysis, we assume that a single
gradient coordinate has an approximately fixed local scale $\sigma$, and use
\[
\frac{g_j^2}{\sigma^2}\overset{d}{\approx}\chi_1^2.
\]
In particular, this implies
\[
\mathbb{E}[g_j^2]\approx \sigma^2,
\]
which we use only to motivate a deterministic reference trajectory. By linearity
of expectation,
\begin{align*}
    \mathbb{E}[v_j]
=
\alpha \sum_{k=1}^j \beta_2^{\,j-k}\,\mathbb{E}[g_k^2]
&\approx
\alpha \sigma^2 \sum_{k=1}^j \beta_2^{\,j-k}\\
&=
\sigma^2(1-\beta_2^j).
\end{align*}

This motivates the reference path
\begin{equation}
\bar v_j := \sigma^2 \phi_j,
\qquad
\phi_j := 1-\beta_2^j.
\label{eq:vbarj-app}
\end{equation}

We index $j$ as the number of post-initialization updates already accumulated
in the EMA. Thus the \emph{next} squared-gradient update is tested against the
current stored state $\bar v_j$. Define
\[
z:=\frac{g^2}{\sigma^2}\sim \chi^2_1,
\qquad
W:=\frac{g^2}{\bar v_j}=\frac{z}{\phi_j}.
\]
Under nearest rounding (NR), stalling occurs when the innovation is smaller than
half a ulp of the current state, i.e.
\[
|W-1|<\hat\rho.
\]
Equivalently,
\[
-\hat\rho < \frac{z}{\phi_j}-1 < \hat\rho
\Longleftrightarrow
\phi_j(1-\hat\rho) < z < \phi_j(1+\hat\rho).
\]
Therefore
\begin{equation}
P_{\mathrm{stall}}^{\mathrm{NR}}(j)
=
F_{\chi^2_1}\!\bigl(\phi_j(1+\hat\rho)\bigr)
-
F_{\chi^2_1}\!\bigl(\max\{0,\phi_j(1-\hat\rho)\}\bigr).
\label{eq:pstall-transient-general-app}
\end{equation}
For the formats of interest in the main text, $\hat\rho\ge 1$, so the lower
endpoint is nonpositive and
\begin{equation}
P_{\mathrm{stall}}^{\mathrm{NR}}(j)
=
F_{\chi^2_1}\!\bigl((1+\hat\rho)\phi_j\bigr).
\label{eq:pstall-transient-app}
\end{equation}

\paragraph{Conservativeness of the ideal-trajectory approximation.}
More generally, if the current state scale is $v=\sigma^2\phi$, then the same
derivation gives
\[
P_{\mathrm{stall}}^{\mathrm{NR}}(\phi)
=
F_{\chi^2_1}\!\bigl((1+\hat\rho)\phi\bigr),
\qquad (\hat\rho\ge 1),
\]
which is monotone increasing in $\phi$. Therefore, whenever the quantized EMA
lags behind the ideal reference trajectory, its instantaneous stalling
probability is smaller than the ideal-trajectory prediction at the same step.
This does not yield a pathwise bound for the full quantized process, since
nearest rounding can also perturb successful updates upward, but it explains
why the ideal-trajectory predictor is typically conservative early in the
startup phase.

\paragraph{Why the naive substitution is incorrect.}
A tempting shortcut is to replace $\hat\rho$ by the scaled quantity
$\hat\rho_j=\hat\rho\,\phi_j$ inside the steady-state formula. This gives a
stalling interval centered at $z=1$ for every $j$, namely
\[
1-\hat\rho\phi_j < z < 1+\hat\rho\phi_j,
\]
which is correct only at steady state ($\phi_j=1$). During the transient, the
relevant ratio is $W=z/\phi_j$, so the correct interval is instead centered at
the \emph{current state scale} $z=\phi_j$, as in
\eqref{eq:pstall-transient-general-app}. This distinction matters especially for
small $\phi_j$, because the $\chi^2_1$ density is largest near zero.

\paragraph{Closed-form startup window.}
Fix a tolerance $P_0\in(0,1)$ and define
\[
j^\star(P_0):=\min\{j:\;P_{\mathrm{stall}}^{\mathrm{NR}}(j)\ge P_0\}.
\]
When $\hat\rho\ge 1$, using \eqref{eq:pstall-transient-app},
\[
F_{\chi^2_1}\!\bigl((1+\hat\rho)\phi_j\bigr)\ge P_0
\Longleftrightarrow
(1+\hat\rho)\phi_j \ge F_{\chi^2_1}^{-1}(P_0).
\]
Hence
\[
\phi_j \ge \phi^\star(P_0),\ 
\phi^\star(P_0):=\frac{F_{\chi^2_1}^{-1}(P_0)}{1+\hat\rho}.
\]
Since $\phi_j=1-\beta_2^j$ is monotone increasing in $j$, the smallest such $j$
is
\begin{equation}
j^\star(P_0)
=
\left\lceil
\frac{\log\!\bigl(1-\phi^\star(P_0)\bigr)}{\log\beta_2}
\right\rceil,
\
\phi^\star(P_0)=\frac{F_{\chi^2_1}^{-1}(P_0)}{1+\hat\rho},
\label{eq:jstar-app}
\end{equation}
defined whenever $\phi^\star(P_0)<1$, i.e.
\[
F_{\chi^2_1}^{-1}(P_0)<1+\hat\rho.
\]

\paragraph{Including the empirical initial floor.}
The closed-form expression above captures only the precision-induced transient.
In practice, the measured stalled fraction also contains an empirical floor
$P_{\mathrm{init}}$, due to effects not modeled by the idealized scalar EMA
analysis, including structural zeros in the gradient and, for narrow-range
formats, dynamic-range crushing under scaled quantization. We therefore use the
phenomenological mixture model
\begin{equation}
P_{\mathrm{total}}(j)
=
P_{\mathrm{init}} + (1-P_{\mathrm{init}})\,P_{\mathrm{stall}}^{\mathrm{NR}}(j).
\label{eq:Ptotal-app}
\end{equation}
For a target $P_0$, this is equivalent to replacing $P_0$ by the effective target
\begin{equation}
P_0^{\mathrm{eff}}
=
\frac{P_0-P_{\mathrm{init}}}{1-P_{\mathrm{init}}},
\label{eq:P0eff-app}
\end{equation}
with the convention that $j^\star(P_0)=0$ whenever $P_0\le P_{\mathrm{init}}$.
Thus the startup window reported in the main text should be interpreted as a
theory-plus-floor predictor rather than a fully first-principles quantity:
\[
j^\star(P_0)
=
j^\star_{\mathrm{ideal}}\!\bigl(P_0^{\mathrm{eff}}\bigr).
\]

\begin{table}[H]
\centering
\small
\label{tab:jstar-main}
\begin{tabular}{l cccc}
\toprule
 & BF16 & FP8 & FP4 \\
\midrule
$P_{\mathrm{init}}$
& $0.17$
& $0.53$
& $0.97$ \\[3pt]

$j^\star(P_0{=}0.5)$
& $76$ / $\mathbf{50}$
& $\mathbf{0}$ / $\mathbf{0}$
& $\mathbf{0}$ / $\mathbf{0}$ \\[3pt]

$j^\star(P_0{=}0.8)$
& $464$ / $\mathbf{360}$
& $15$ / $\mathbf{15}$
& $\mathbf{0}$ / $\mathbf{0}$ \\[3pt]

$j^\star(P_0{=}0.9)$
& $1051$ / $\mathbf{700}$
& $36$ / $\mathbf{43}$
& $\mathbf{0}$ / $\mathbf{0}$ \\[3pt]

$j^\star(P_0{=}0.95)$
& $3044$ / --
& $61$ / $\mathbf{79}$
& $\mathbf{0}$ / $\mathbf{0}$ \\
\bottomrule
\end{tabular}
\caption{Startup window under nearest rounding for $\beta_2=0.999$ after
incorporating the empirical floor $P_{\mathrm{init}}$. Theory uses
\eqref{eq:jstar-corrected-main}--\eqref{eq:P0eff-main}; empirical values are
measured in LLM pretraining. ``$\mathbf{0}$'' means the target is already met
immediately because $P_0\le P_{\mathrm{init}}$. ``--'' means the threshold was
not reached empirically within the observation window.}
\vspace{-1em}
\end{table}

\paragraph{Startup-window values.}
\Cref{tab:jstar-main} gives the corresponding numerical startup-window values.
The main trend is consistent with \Cref{fig:staleness}: BF16 retains the
longest responsive window, FP8 remains responsive only over a much shorter
range, and in FP4 the \emph{total} stalled fraction already exceeds practical
thresholds immediately after reset once the empirical floor $P_{\mathrm{init}}$
is incorporated.

\begin{figure*}[t]
    \centering
    \includegraphics[width=\linewidth]{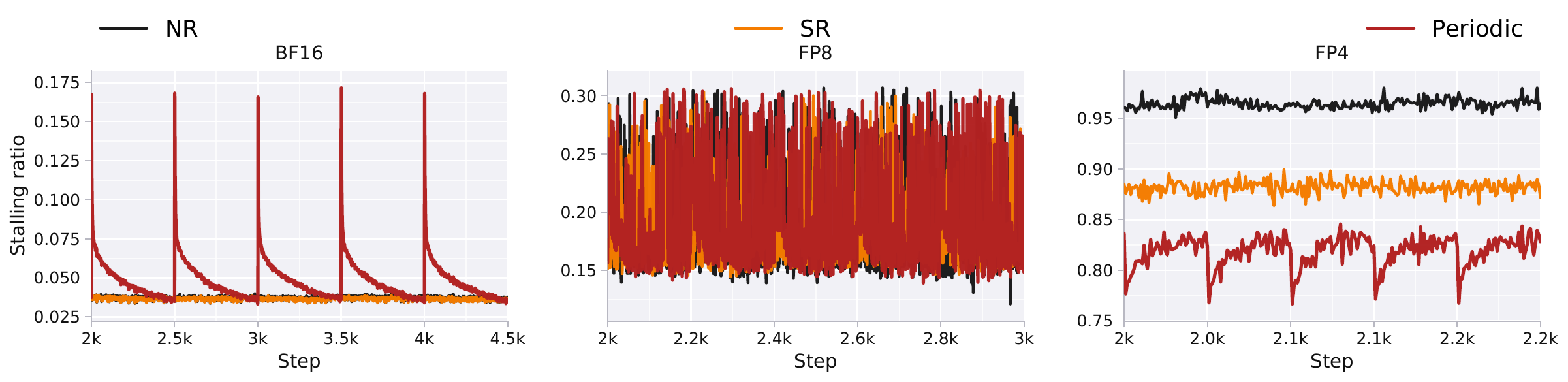}
    \vspace{-2em}
    \caption{Measured stalling fraction for the first moment during pretraining of the 100M-parameter LLaMA model on C4.}
    \vspace{-1em}
    \label{fig:first_moment_staleness}
\end{figure*}

\subsection{Empirical observations of stalling in the first moment}
\label{app:first_moment}

\Cref{fig:first_moment_staleness} shows that the first moment exhibits the same
basic qualitative pattern as the second: the stalled fraction is already
nonzero after initialization and then quickly approaches a steady regime.
What differs is the severity. In BF16, the steady-state stalled fraction is
very small, confirming that first-moment stalling is largely negligible in
this regime. Moreover, the initial floor $P_{\mathrm{init}}$ is slightly higher
than the steady-state level, which explains the mild non-monotonic behavior and
suggests that repeatedly resetting the first moment is not especially useful.
In FP8, the steady-state stalled fraction is still relatively low, and is close
to the initial floor; accordingly, resetting the first moment has little effect,
either positive or negative. In FP4, by contrast, the steady-state stalled
fraction becomes substantially larger and clearly harmful, while remaining above
$P_{\mathrm{init}}$, so resetting can recover part of the performance lost to
persistent first-moment stalling. Finally, as in the second moment, stochastic
rounding lowers the steady-state stalled fraction, but its practical importance
depends strongly on the underlying precision regime.

\subsection{Derivation of the reset-period formula}
\label{app:reset-heuristic}

This section derives the theory-guided rule used in
\Cref{sec:resets-rewrite}. The goal is not to identify a uniquely optimal
reset period, but rather to pre-select a useful operating region: long enough
to let the EMA accumulate statistical information, but not so long that the
\emph{reset-sensitive} part of stalling dominates a substantial portion of the
cycle. Compared to a purely endpoint-based rule, the rule below
incorporates two empirical facts: (i) moderate second-moment staleness is often
harmless, and (ii) good reset periods typically form a broad window rather than
a sharp optimum.

\paragraph{Step 1: normalize the controllable stalling buildup.}
After a reset, the total measured stalled fraction at step $j$ is modeled as
\[
P_{\mathrm{total}}(j)
=
P_{\mathrm{init}} + (1-P_{\mathrm{init}})\,P_{\mathrm{stall}}(j),
\]
where $P_{\mathrm{init}}$ is the baseline floor and $P_{\mathrm{stall}}(j)$ is
the precision-induced stalling on the responsive subset. Since
$P_{\mathrm{init}}$ is already present immediately after reset and is not
removed by resetting, it should not determine the reset period. We therefore
define the normalized stalling progress
\begin{equation}
S(j)
:=
\frac{P_{\mathrm{total}}(j)-P_{\mathrm{init}}}
     {P_{\mathrm{total}}^{\mathrm{ss}}-P_{\mathrm{init}}}.
\label{eq:Sj-def}
\end{equation}
Using
\[
P_{\mathrm{total}}^{\mathrm{ss}}
=
P_{\mathrm{init}} + (1-P_{\mathrm{init}})\,P_{\mathrm{stall}}^{\mathrm{ss}},
\]
this simplifies to
\begin{equation}
S(j)
=
\frac{P_{\mathrm{stall}}(j)}{P_{\mathrm{stall}}^{\mathrm{ss}}}.
\label{eq:Sj-simplified}
\end{equation}
Under the transient model, $S(j)$ starts at $0$ and approaches $1$ as
$j\to\infty$. The key point is that this normalization removes the
reset-insensitive floor $P_{\mathrm{init}}$ and isolates the part of stalling
that accumulates within a cycle and can be mitigated by resetting.

\paragraph{Step 2: measure cycle-averaged excess staleness.}
A rule based only on the endpoint $S(K)$ is often too conservative: it
penalizes a cycle as soon as the \emph{last} step becomes highly stalled, even
if much of the cycle remains in a regime where training is still robust. To
reflect cumulative exposure to harmful staleness, we average over the cycle and
count only the excess above a tolerance level $s_0\in[0,1)$:
\begin{align}
\bar S_{s_0}(K)
&:=
\frac{1}{K}
\sum_{j=1}^{K}
\left[
\frac{S(j)-s_0}{1-s_0}
\right]_+,\label{eq:avgS-threshold}
\\
[x]_+&:=\max(x,0).
\end{align}
This quantity is $0$ when the entire cycle stays below the tolerated staleness
level $s_0$, and approaches $1$ only when a substantial fraction of the cycle
is spent near fully saturated controllable stalling. Thus $s_0$ encodes the
empirically observed tolerance to moderate staleness: larger $s_0$ delays the
point at which staleness begins to count against a longer cycle.

\paragraph{Step 3: quantify the remaining statistical error of the EMA.}
Consider the bias-corrected EMA after $K$ steps, where the local age of the EMA
is reset at the reset boundary together with the moment state itself. Its
normalized weights are
\[
w_t^{(K)}
=
\frac{(1-\beta_2)\beta_2^{K-t}}{1-\beta_2^K},
\qquad t=1,\dots,K.
\]
The corresponding effective sample size is
\begin{equation}
N_{\mathrm{stat}}(K)
=
\frac{1}{\sum_{t=1}^K (w_t^{(K)})^2}
=
\frac{(1+\beta_2)(1-\beta_2^K)}
     {(1-\beta_2)(1+\beta_2^K)}.
\label{eq:Nstat}
\end{equation}
Its infinite-horizon limit is
\[
N_{\mathrm{stat}}^{\infty}
=
\frac{1+\beta_2}{1-\beta_2}.
\]
We define the normalized remaining statistical error as
\begin{equation}
E(K)
:=
1-\frac{N_{\mathrm{stat}}(K)}{N_{\mathrm{stat}}^{\infty}}
=
\frac{2\beta_2^K}{1+\beta_2^K}.
\label{eq:EK}
\end{equation}
This quantity decreases monotonically from $1$ to $0$ as the EMA accumulates
samples, and measures how much averaging benefit is still left to gain by
extending the current cycle.

\paragraph{Step 4: define the crossing time.}
We choose the reset period as the first cycle length for which the
cycle-averaged excess staleness matches the remaining statistical error:
\begin{equation}
K^\star
:=
\min\bigl\{
K\ge 1:\ \bar S_{s_0}(K)\ge E(K)
\bigr\}.
\label{eq:crossing-heuristic}
\end{equation}
This yields a direct tradeoff. Before the crossing, the EMA is still gaining
substantial statistical accuracy relative to the harmful staleness accumulated
within the cycle. After the crossing, the additional gain from a longer cycle
is increasingly outweighed by the time spent in a highly stalled regime.

\paragraph{Step 5: substitute the transient stalling law.}
For the precision regimes considered in the main text, $\hat\rho\ge 1$, so the
transient NR approximation from \Cref{app:jstar} gives
\[
P_{\mathrm{stall}}(j)
\approx
F_{\chi^2_1}\!\bigl((1+\hat\rho)(1-\beta_2^j)\bigr),
\
P_{\mathrm{stall}}^{\mathrm{ss}}
\approx
F_{\chi^2_1}(1+\hat\rho).
\]
Substituting into \eqref{eq:Sj-simplified} yields
\begin{equation}
S(j)
\approx
\frac{
F_{\chi^2_1}\!\bigl((1+\hat\rho)(1-\beta_2^j)\bigr)
}{
F_{\chi^2_1}(1+\hat\rho)
}.
\label{eq:Sj-transient}
\end{equation}
Therefore the cycle-averaged excess staleness becomes
\begin{equation}
\bar S_{s_0}(K)
\approx
\frac{1}{K}
\sum_{j=1}^{K}
\left[
\frac{
\displaystyle
\frac{
F_{\chi^2_1}\!\bigl((1+\hat\rho)(1-\beta_2^j)\bigr)
}{
F_{\chi^2_1}(1+\hat\rho)
}
-s_0
}{
1-s_0
}
\right]_+.
\label{eq:avgS-transient}
\end{equation}
Combining \eqref{eq:avgS-transient} with \eqref{eq:EK}, the reset period is
the smallest integer $K$ such that
\begin{equation}
\frac{1}{K}
\sum_{j=1}^{K}
\left[
\frac{
\displaystyle
\frac{
F_{\chi^2_1}\!\bigl((1+\hat\rho)(1-\beta_2^j)\bigr)
}{
F_{\chi^2_1}(1+\hat\rho)
}
-s_0
}{
1-s_0
}
\right]_+
\ge
\frac{2\beta_2^K}{1+\beta_2^K}.
\label{eq:crossing-K}
\end{equation}
The left-hand side is nondecreasing in $K$, since it is the average of a
nondecreasing sequence, while the right-hand side is strictly decreasing and
tends to $0$. Moreover, because $S(j)\to 1$, the left-hand side tends to $1$.
Hence the crossing in \eqref{eq:crossing-K} exists and is unique.

\paragraph{Numerical values.}
For $\beta_2=0.999$ and the effective precision ratios used in the main text,
\[
\hat\rho_{\mathrm{BF16}}=2.71,\quad
\hat\rho_{\mathrm{FP8}}=43.3,\quad
\hat\rho_{\mathrm{FP4}}=86.6,
\]
we solve \eqref{eq:crossing-K} numerically. For the tolerance level used in the
main text, $s_0=0.6$, this gives
\[
K^\star_{\mathrm{BF16}}\approx 1116,\quad
K^\star_{\mathrm{FP8}}\approx 320,\quad
K^\star_{\mathrm{FP4}}\approx 224.
\]
Table~\ref{tab:reset-periods-thresholded} also reports nearby values of $s_0$
to show that the rule is stable across a reasonable range of tolerance
levels. Overall, the predicted reset periods lie in the same broad regime as
the empirically robust windows observed in our experiments, namely roughly
$500$--$1500$ for BF16, $300$--$600$ for FP8, and $50$--$250$ for FP4.

\begin{table}[H]
\centering
\scriptsize
\begin{tabular}{lcccccc}
\toprule
& & \multicolumn{3}{c}{$K^\star$ for tolerance level $s_0$} & Empirical\\
\cmidrule(lr){3-5}
Format & $\hat\rho$ & $0.5$ & $0.6$ & $0.7$ &  window \\
\midrule
BF16 & $2.71$ & $\approx 1004$ & $\approx 1116$ & $\approx 1262$ & $500$--$1500$ \\
FP8  & $43.3$ & $\approx 295$  & $\approx 320$  & $\approx 351$  & $300$--$600$ \\
FP4  & $86.6$ & $\approx 206$  & $\approx 224$  & $\approx 246$  & $50$--$250$ \\
\bottomrule
\end{tabular}
\caption{Reset periods from the cycle-averaged thresholded rule
\eqref{eq:crossing-K} under NR with $\beta_2=0.999$. We report several
tolerance levels $s_0$ to show the sensitivity of the rule to the tolerated
amount of reset-sensitive staleness.}
\label{tab:reset-periods-thresholded}
\vspace{-1em}
\end{table}

In the main text we use $s_0=0.6$, which places the predicted periods near the
middle of the empirically robust range while remaining stable to moderate
changes in the tolerance level.

\subsection{Adaptive resetting policies}
\label{sec:adaptive-restarts}

While periodic resets are simple and effective, one can also design
\emph{adaptive} reset rules by monitoring stalling online during training.
This avoids committing to a fixed period in advance. Our adaptive rule follows
the same principle as the fixed periodic rule in
\Cref{prop:restart-heuristic}, but evaluates it online rather than
precomputing a reset interval.

Concretely, at each iteration we estimate the empirical stalled fraction
$P_{\mathrm{stall}}(k)$ by comparing the newly quantized state with the
previously stored state before writeback. We then form the same
cycle-averaged excess staleness quantity used in the periodic rule, but
replace the model-based transient predictor by the observed stalled fraction.
In the adaptive version, we fix $P_{\mathrm{stall}}^{\mathrm{ss}}=1$ and
$s_0=0.6$, and trigger a reset once
\begin{equation}
\bar S_{s_0}(k) - E(k) > 0.
\end{equation}
Intuitively, this condition detects when the observed buildup of stalling
within the current cycle outweighs the remaining statistical benefit of
continuing the EMA.

Thus, the adaptive policy implements the same reset principle as the periodic
policy, but replaces a fixed reset interval with an online test based on
the actual behavior of the quantized EMA states. In practice, periodic resets
provide a simple default strategy, while adaptive resetting offers a lightweight
extension that can respond to the realized stalling dynamics during training.

\subsection{Experimental Setting}
\label{app:training-details}

We introduce the LLaMA architecture and training setup used in our experiments.
All models use a maximum sequence length of 256, a global token batch size of
131K, and BF16 model weights and activations. We apply gradient clipping with
maximum norm 1.0, linear warmup over the first $10\%$ of training steps, and
cosine decay to $10\%$ of the peak learning rate. 

\begin{table*}[t]
\centering
\small
\begin{tabular}{lccccccc}
\toprule
Params & Hidden & Intermediate & Heads & Layers & Steps & Tokens & LR \\
\midrule
60M  & 512  & 1376 & 8  & 8  & 10K & 1.3B & 0.001 \\
100M & 640  & 1708 & 10 & 12 & 20K & 2.6B & 0.001 \\
350M & 1024 & 2736 & 16 & 24 & 60K & 7.8B & 0.0004 \\
\bottomrule
\end{tabular}
\caption{Model configurations used in our experiments. Data amounts are given in tokens.}
\label{tab:model-configs}
\end{table*}
\begin{figure*}[t]
\centering
\includegraphics[width=\linewidth]{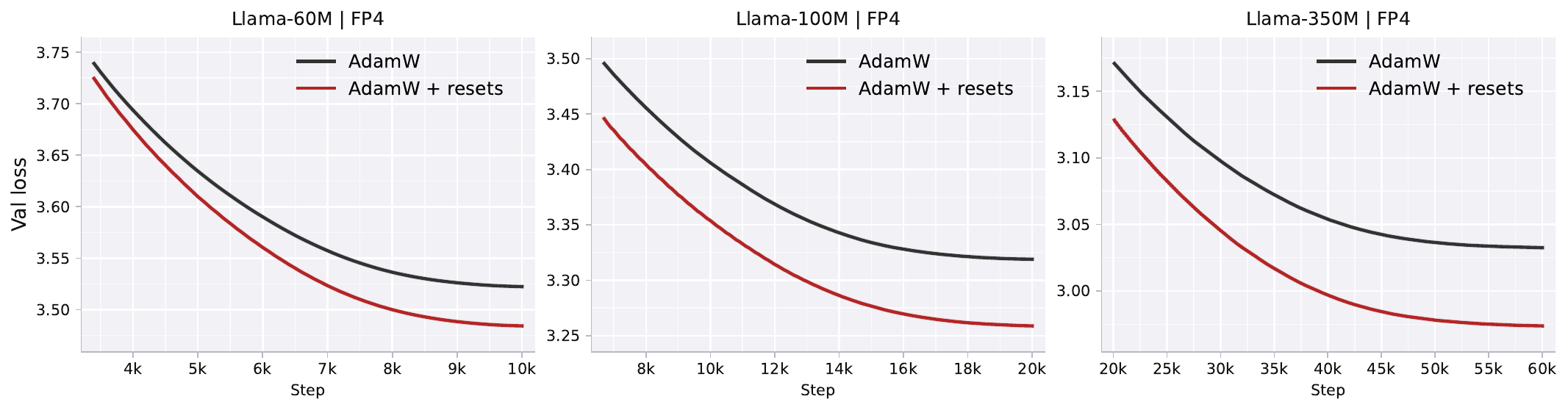}
\caption{Validation loss curves for the three model scales with FP4 optimizer-state storage.}
\label{fig:training-curves-fp4}
\end{figure*}
We train all models on C4~\cite{raffel2020exploring} using AdamW with
$\beta_1=0.9$, $\beta_2=0.999$, and AdamW epsilon
$\epsilon_{\mathrm{adam}}=10^{-6}$. The peak learning rates used for each model
are reported in \Cref{tab:model-configs} and are kept fixed across precision
regimes. SOAP-specific parameters for the ablation  include the preconditioner update frequency every 10 steps, maximum preconditioner dimension 10,000, and preconditioner EMA coefficient $\beta=0.999$.

We evaluate three optimizer-state precision settings: BF16, FP8 (E4M3 for both
moments with per-tensor scaling), and FP4 (E2M1 for the first moment and
unsigned E2M2 for the second moment with block size 128). Model weights and
activations remain in BF16.

Because current hardware does not provide native support for all of the
low-precision state formats and rounding modes considered here, we simulate
quantized optimizer-state storage in software. Concretely, optimizer states are
dequantized to high precision for the update, then quantized and stored in the target
format after each step, using either nearest or stochastic rounding depending
on the experiment.

When enabled, stochastic rounding is applied to all quantized states. Periodic
EMA resets are applied to both moments unless stated otherwise. The reset
period for the second moment is chosen using
\Cref{eq:reset-heuristic-main}, yielding periods of 1000 (BF16), 300 (FP8),
and 200 (FP4). The first moment uses the same period unless otherwise
specified.

\subsection{Additional Experimental Results}
\label{app:additional-results}

We report additional experiments evaluating whether the observed benefits
persist at longer horizons, larger model scales, under alternative
low-memory optimizers, and beyond pretraining loss.

\paragraph{End-to-end pretraining at 1B parameters.}
In addition to the short diagnostic runs used to test the stalling theory, we
perform end-to-end pretraining of a 1B-parameter LLaMA model for 120K steps,
corresponding to approximately 16B tokens. Stochastic rounding and resets
improve the final loss under both BF16 and FP8 optimizer-state storage
(\Cref{tab:app-one-billion}). This shows that the benefits persist in a
substantially larger and longer pretraining setting.

\begin{table}[H]
\centering
\small
\begin{tabular}{lcc}
\toprule
States & AdamW & AdamW+SR+resets \\
\midrule
BF16 & 2.91 & \textbf{2.77} \\
FP8  & 2.89 & \textbf{2.78} \\
\bottomrule
\end{tabular}
\caption{End-to-end pretraining of LLaMA-1B for 120K steps
(approximately 16B tokens). Lower loss is better.}
\label{tab:app-one-billion}
\end{table}

\paragraph{Longer-horizon pretraining.}
We also extend LLaMA-100M training from the original 20K-step setting to 60K
steps (\Cref{tab:app-long-horizon}).

\begin{table}[H]
\centering
\small
\begin{tabular}{lcc}
\toprule
States & AdamW & AdamW+SR+resets \\
\midrule
BF16 & 3.162 & \textbf{3.044} \\
FP8  & 3.223 & \textbf{3.098} \\
\bottomrule
\end{tabular}
\caption{Long-horizon LLaMA-100M pretraining for 60K steps.
Lower loss is better.}
\label{tab:app-long-horizon}
\end{table}

\paragraph{Zero-shot evaluation.}
To test whether the pretraining-loss improvements are reflected in downstream
behavior, we evaluate the 350M checkpoints without task-specific fine-tuning across 5 seeds (\Cref{tab:app-zero-shot}).
\begin{table*}[t]
\centering
\small
\begin{tabular}{llrrrrrr}
\toprule
States & Method & HellaSwag & PIQA & ARC-E & ARC-C & WinoGrande & Avg. \\
\midrule
BF16 & AdamW        & 31.4 & 61.6 & 30.9 & 25.3 & 49.0 & $39.6\pm0.0$ \\
BF16 & AdamW+SR+resets & 33.1 & 63.6 & 32.4 & 25.5 & 50.0 & $\mathbf{40.9\pm0.1}$ \\
\midrule
FP8 & AdamW         & 30.7 & 62.2 & 30.0 & 25.0 & 48.8 & $39.3\pm0.1$ \\
FP8 & AdamW+SR+resets  & 32.5 & 63.9 & 31.9 & 25.6 & 50.7 & $\mathbf{40.9\pm0.1}$ \\
\midrule
FP4 & AdamW         & 31.4 & 62.8 & 29.9 & 24.0 & 48.2 & $39.3\pm0.2$ \\
FP4 & AdamW+SR+resets  & 32.8 & 63.4 & 31.7 & 25.1 & 49.8 & $\mathbf{40.6\pm0.1}$ \\
\bottomrule
\end{tabular}
\caption{Zero-shot evaluation of the 350M checkpoints. Scores are
percentages. Avg. is the mean over the listed tasks, with standard deviation
over five evaluation seeds.}
\label{tab:app-zero-shot}
\end{table*}

\paragraph{Comparison with low-memory optimizers.}
The FP4 AdamW experiments in the main text use the established 4-bit
optimizer-state method of \citet{4bit_states}; thus, our intervention is
complementary to an existing low-precision optimizer rather than an alternative
quantizer. We additionally compare with blockwise 8-bit
AdamW~\cite{dettmers20218} and Adam-mini~\cite{zhang2024adam}. Methods are
grouped according to their optimizer-state-memory budget on LLaMA-100M.
Across these budgets, stochastic rounding and resets consistently improve the
corresponding low-memory baseline (\Cref{tab:app-low-memory}).

\begin{table*}[t]
\centering
\small
\begin{tabular}{llll}
\toprule
Budget & Method & State configuration & Loss \\
\midrule
BF16
& AdamW & BF16 & 3.27 \\
& AdamW+SR+resets & BF16 & 3.21 \\
& Adam-mini & FP32 & 3.24 \\
& Adam-mini+resets & FP32 & 3.24 \\
\midrule
FP8
& 8-bit AdamW & Dettmers blockwise INT8 & 3.24 \\
& 8-bit AdamW+SR+resets & Dettmers blockwise INT8 & 3.21 \\
& Adam-mini & BF16 & 3.29 \\
& Adam-mini+SR+resets & BF16 & 3.24 \\
\midrule
FP4
& FP4 AdamW & Li et al.\ FP4 & 3.32 \\
& FP4 AdamW+SR+resets & Li et al.\ FP4 & 3.26 \\
& Adam-mini & FP8 & 3.40 \\
& Adam-mini+SR+resets & FP8 & 3.24 \\
\bottomrule
\end{tabular}
\caption{Comparison of low-memory optimizers on LLaMA-100M, grouped
by optimizer-state-memory budget. The budget labels refer to the memory group,
not necessarily to the precision of every stored entry. Lower loss is better.}
\label{tab:app-low-memory}
\end{table*}

\paragraph{Runtime and optimizer-state memory.}
Our efficiency claim concerns optimizer-state memory rather than measured
throughput speedups. Native BF16 storage halves optimizer-state memory relative
to FP32 without reducing throughput, while periodic resets add negligible
overhead. The visible slowdown in our experiments comes from the simulated
quantization/dequantization operations, which are not implemented using fused
native kernels (\Cref{tab:app-runtime}).

\begin{table*}[t!]
\centering
\small
\begin{tabular}{llrrr}
\toprule
Implementation & Resets & Sec./step & Tokens/sec & State memory (GiB) \\
\midrule
FP32 native & No & 3.888 & 25,718 & 0.746 \\
BF16 native & No & 3.862 & 25,895 & 0.373 \\
BF16 native & $K=100$ & 3.859 & 25,915 & 0.373 \\
BF16 simulated Quant/Dequant& No & 4.120 & 24,268 & 0.373 \\
BF16 simulated Quant/Dequant + simulated SR & No & 4.130 & 24,209 & 0.373 \\
BF16 simulated Quant/Dequant + simulated SR & $K=100$ & 4.068 & 24,581 & 0.373 \\
\bottomrule
\end{tabular}
\caption{BF16 runtime and optimizer-state-memory measurements on 100K token batches. K=100 is a deliberately frequent reset overhead.}
\label{tab:app-runtime}
\end{table*}

\paragraph{Low-precision model weights.}
Low-precision weights can alter the gradient distribution and therefore affect
the initial stalled fraction. As a controlled check, we repeat the stalling
diagnostic with native PyTorch FP8 model weights. With BF16 optimizer states,
$P_{\mathrm{init}}/P_{\mathrm{ss}}$ changes from $0.271/0.969$ to
$0.258/0.989$; with FP8 optimizer states, it changes from $0.423/0.999$ to
$0.406/0.997$. Thus, while the initial floor changes
slightly, the steady-state optimizer-state stalling behavior is essentially
unchanged in this FP8-weight setting.

\paragraph{Effect of scaling granularity.}
Finally, we evaluate FP4 optimizer-state storage with per-tensor scaling
instead of the blockwise scaling used in the main experiments. Per-tensor
scaling produces a substantially coarser effective grid and correspondingly
worse performance: even with SR and resets, the loss reaches only approximately
$3.8$, compared with $3.45$--$3.55$ under blockwise scaling. Nevertheless,
resets remain important: without them, training fails to converge
(\Cref{fig:training-curves-fp4-pertensor}). This shows that resets complement,
but do not replace, an appropriate quantization and scaling strategy.

\begin{figure}[h]
\centering
\includegraphics[width=0.6\linewidth]{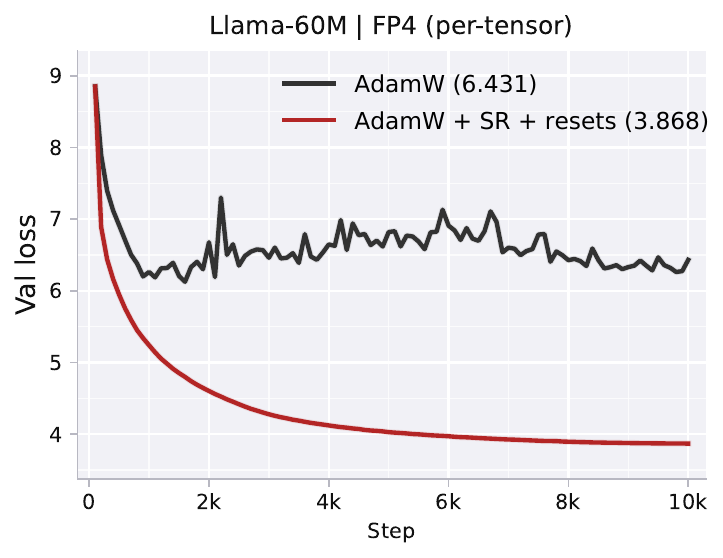}
\caption{Validation loss curves for LLaMA-60M with FP4 optimizer-state storage under per-tensor scaling.}
\label{fig:training-curves-fp4-pertensor}
\end{figure}

\subsection{Algorithms}

Here $s_t$ may denote a single EMA state or a tuple of optimizer states; when
multiple states are present, $\Phi,\Psi,\Q$, and $\D$ are understood
componentwise.

\begin{algorithm}[H]
\footnotesize
\caption{Stateful Optimizer with Quantized State Storage and Periodic Resets}
\label{alg:general-quant-restarts}
\begin{algorithmic}[1]
\Require Update maps $\Phi,\Psi$; hyperparameters $h$; quantizer $\Q$ (NR or SR);
dequantizer $\D$; reset period $K$; initial state $s_{\mathrm{init}}$ (often $0$);
(optional) resettable schedule state $\xi_{\mathrm{init}}$.
\State Initialize $\theta_0$;\; $\bar s_0 \leftarrow \Q(s_{\mathrm{init}})$;\; $k\leftarrow 0$;\;
(optional) $\xi\leftarrow \xi_{\mathrm{init}}$
\For{$t=1,2,\ldots$}
    \State $g_t \leftarrow \nabla f_t(\theta_{t-1})$
    \State $s^{\mathrm{hp}}_{t-1} \leftarrow \D(\bar s_{t-1})$ \Comment{load quantized state}
    \State $k \leftarrow k + 1$ \Comment{local step within the current reset cycle}
    \State $s^{\mathrm{hp}}_t \leftarrow \Phi\!\bigl(s^{\mathrm{hp}}_{t-1},\, g_t,\, \theta_{t-1};\, h,\, k,\, \xi\bigr)$
    \State $\Delta\theta_t \leftarrow \Psi\!\bigl(s^{\mathrm{hp}}_t,\, g_t,\, \theta_{t-1};\, h,\, k,\, \xi\bigr)$
    \State $\theta_t \leftarrow \theta_{t-1} + \Delta\theta_t$ \Comment{parameter update in higher precision}
    \State $\bar s_t \leftarrow \Q\!\bigl(s^{\mathrm{hp}}_t\bigr)$ \Comment{requantize state for storage}
    \If{$k = K$}
        \State $\bar s_t \leftarrow \Q(s_{\mathrm{init}})$ \Comment{often $\bar s_t \leftarrow 0$}
        \State $k \leftarrow 0$
        \State (optional) $\xi \leftarrow \xi_{\mathrm{init}}$ \Comment{reset bias correction / schedules}
    \EndIf
\EndFor
\end{algorithmic}
\end{algorithm}

\begin{algorithm}[t]
\footnotesize
\caption{Stateful Optimizer with Quantized State Storage and Adaptive Resets}
\label{alg:adaptive-restarts}
\begin{algorithmic}[1]
\Require Update maps $\Phi,\Psi$; hyperparameters $h$; quantizer $\Q$; dequantizer $\D$;
decay $\beta_2$; reset tolerance $s_0$; steady-state stalled fraction
$P_{\mathrm{stall}}^{\mathrm{ss}}$ (we use $P_{\mathrm{stall}}^{\mathrm{ss}}=1$ in practice);
initial state $s_{\mathrm{init}}$ (often $0$); optional resettable schedule state
$\xi_{\mathrm{init}}$.
\State Initialize parameters $\theta_0$
\State Initialize quantized state $\bar s_0 \leftarrow \Q(s_{\mathrm{init}})$
\State Initialize cycle counter $k\leftarrow 0$
\State Initialize accumulated excess staleness $A\leftarrow 0$
\State (optional) initialize schedule state $\xi\leftarrow \xi_{\mathrm{init}}$
\For{$t=1,2,\ldots$}
    \State $g_t \leftarrow \nabla f_t(\theta_{t-1})$
    \State $s_{t-1}^{\mathrm{hp}} \leftarrow \D(\bar s_{t-1})$
    \State $k \leftarrow k+1$
    \State $s_t^{\mathrm{hp}} \leftarrow \Phi(s_{t-1}^{\mathrm{hp}}, g_t, \theta_{t-1}; h, k, \xi)$
    \State $\Delta\theta_t \leftarrow \Psi(s_t^{\mathrm{hp}}, g_t, \theta_{t-1}; h, k, \xi)$
    \State $\theta_t \leftarrow \theta_{t-1} + \Delta\theta_t$
    \State $\bar s_t \leftarrow \Q(s_t^{\mathrm{hp}})$
    \State Compute the empirical stalled fraction
    \[
    P_{\mathrm{stall}}(k)
    \leftarrow
    \frac{1}{d}\sum_{i=1}^d \mathbf{1}\!\left[(\bar s_t)_i=(\bar s_{t-1})_i\right]
    \]
    where $d$ is the number of entries in the state tensor
    \State Compute normalized stalling progress
    \[
    S(k)\leftarrow \frac{P_{\mathrm{stall}}(k)}{P_{\mathrm{stall}}^{\mathrm{ss}}}
    \]
    \State Update accumulated excess staleness
    \[
    A \leftarrow A + \left[\frac{S(k)-s_0}{1-s_0}\right]_+
    \]
    \State Compute cycle-averaged excess staleness
    \[
    \bar S_{s_0}(k)\leftarrow \frac{A}{k}
    \]
    \State Compute remaining statistical error
    \[
    E(k)\leftarrow \frac{2\beta_2^k}{1+\beta_2^k}
    \]
    \If{$\bar S_{s_0}(k)\ge E(k)$}
        \State \textbf{reset:} $\bar s_t \leftarrow \Q(s_{\mathrm{init}})$
        \State $k\leftarrow 0$
        \State $A\leftarrow 0$
        \State (optional) $\xi\leftarrow \xi_{\mathrm{init}}$
    \EndIf
\EndFor
\end{algorithmic}
\end{algorithm}
\end{document}